\documentclass[12pt]{article}

\usepackage{amsmath}
\usepackage{amsmath, amssymb, epsfig, amsthm}
\usepackage{mathtools}
\usepackage{mathrsfs}
\usepackage{color}
\usepackage[table]{xcolor}
\usepackage{url}
\usepackage{algorithm}
\usepackage{algpseudocode}
\usepackage{graphicx}
\usepackage{caption}
\usepackage{comment}

\usepackage{enumitem}
\usepackage{color}
\usepackage{subfig}
\usepackage{wrapfig}
\usepackage[scaled]{helvet} 
\usepackage[margin=1in]{geometry}
\usepackage{dsfont,bm,bbm}
\usepackage{authblk}
\usepackage[title, toc, page]{appendix}

\newcommand{\mtx}[1]{\bm{#1}}

\newtheorem{lemma}{Lemma}

\newtheorem{proposition}{Proposition}
\newtheorem{theorem}{Theorem}[section]
\newtheorem{remark}{Remark}
\newtheorem{corollary}[theorem]{Corollary}

\newcommand{\U}{{\mathbf{U}}}
\newcommand{\V}{{\mathbf{V}}}
\newcommand{\X}{{\mathbf{X}}}
\newcommand{\D}{{\mathbf{D}}}
\newcommand{\M}{{\mathbf{M}}}
\newcommand{\Z}{{\mathbf{Z}}}
\newcommand{\bW}{{\mathbf{W}}}
\newcommand{\Zopt}{{\mathbf{Z}^*}}
\newcommand{\Wopt}{{\mathbf{W}^*}}
\newcommand{\Id}{{\mathbf{Id}}}

\newcommand{\bSigma}{{\mathbf{\Sigma}}}
\def \D {\mtx{D}}

\newcommand{\G}{{\mathbf{G}}}
\newcommand{\B}{{\mathbf{B}}}
\newcommand{\Q}{{\mathbf{Q}}}
\newcommand{\La}{{\mathbf{\Lambda}}}

\newcommand{\z}{{\mathbf{z}}}
\newcommand{\p}{{\mathbf{p}}}

\newcommand{\bP}{{\mathbf{P}}}
\newcommand{\E}{{\mathbf{E}}}

\newcommand{\col}[1]{#1_{(\cdot,j)}}
\newcommand{\row}[1]{#1_{(i,\cdot)}}
\newcommand{\rowj}[1]{#1_{(j,\cdot)}}
\newcommand{\entry}[1]{{#1}_{(i,j)}}

\DeclareMathOperator*{\argmin}{arg\,min}

\newcommand{\subscript}[2]{$#1 _ #2$}

\begin{document}

\title{Analysis of Fast Structured Dictionary Learning\footnote{This article has been accepted for publication in Information and Inference Published by Oxford University Press.}}

\author[1]{Saiprasad~Ravishankar}
\author[2]{Anna Ma}
\author[3]{Deanna Needell}

\affil[1]{Department of Electrical Engineering and Computer Science, University of Michigan, Ann Arbor}
\affil[2]{Institute of Mathematical Science, Claremont Graduate University}
\affil[3]{Department of Mathematics, University of California Los Angeles}

\maketitle

\begin{abstract}
{Sparsity-based models and techniques have been exploited in many signal processing and imaging applications.
Data-driven methods based on dictionary and sparsifying transform learning enable learning rich image features from data, and can outperform analytical models. 
In particular, alternating optimization algorithms have been popular for learning such models.
In this work, we focus on alternating minimization for a specific structured unitary sparsifying operator learning problem, and provide a convergence analysis. While the algorithm converges to the critical points of the problem generally, our analysis establishes under mild assumptions, the local linear convergence of the algorithm to the underlying sparsifying
model of the data. Analysis and numerical simulations show that our assumptions hold for standard probabilistic data models. In practice, the algorithm is robust to initialization.}
{Sparse representations, Dictionary learning, Transform learning, Alternating minimization, Convergence guarantees, Generative models, Fast algorithms.}
\\
2000 Math Subject Classification: 68T05.
\end{abstract}

\section{Introduction} \label{sec1}
Various models of signals and images have been exploited in signal processing and imaging applications, such as dictionary and sparsifying transform models, tensor models, and manifold models. Wavelets and other analytical sparsifying transforms have been used in compression standards \cite{jpg2}, denoising, and magnetic resonance image (MRI) reconstruction from compressive measurements \cite{lustig}. 
While these approaches used \emph{fixed} or analytical image models that are independent of the input data, there has been a rising interest in data-dependent or data-driven models.
Learned models may outperform analytical models 
in various applications. For example, learned dictionaries and sparsifying transforms work well in applications such as denoising \cite{elad2}, in-painting \cite{Mai,studer2012dictionary}, and medical image reconstruction \cite{xu:12:ldx}.This work focuses on analyzing the convergence behavior of a structured (unitary) sparsifying transform  learning algorithm and investigates its ability to recover underlying data models.
In the following, we present some background on dictionary and sparsifying operator learning, before discussing the specific learning problem and algorithm, and our contributions.

\subsection{Background}

Signals can be modeled as sparse in different ways such as in a synthesis dictionary or in a transform domain. In particular, the synthesis dictionary model represents a given signal $ \mathbf{p}  \in \mathbb{R}^{n}$ as $\p \approx \D\z \in \mathbb{R}^n$ with $\mathbf{D}  \in \mathbb{R}^{n \times J} $ denoting the synthesizing dictionary and $\mathbf{z} \in \mathbb{R}^{J} $ denoting the sparse code, i.e., $\left \| \mathbf{z} \right \|_{0}\ll n$ with the $\ell_{0}$ ``norm" counting the number of non-zero vector entries. The synthesis dictionary model is often refereed to as a union of (low-dimensional) subspaces model for signals, wherein different signals may be approximately spanned by different subsets of dictionary columns or atoms. Finding the optimal sparse representation for a signal in the synthesis dictionary model involves solving the well-known synthesis \emph{sparse coding} problem.\footnote{For example, one may minimize $\left \| \mathbf{p}-\mathbf{D}\mathbf{z} \right \|_{2}^{2}$ with respect to $\mathbf{z}$ subject to  $\left \| \mathbf{z} \right \|_{0}\leq s$, where $s$ denotes a set sparsity level, or an alternative version of this problem.} This problem is known to be NP-hard in general \cite{npb} and numerous algorithms exist for approximating the solution to the sparse coding problem \cite{pati, chen2,  befro, Needell2, wei} that provide the correct solution under certain conditions. On the other hand, the sparsifying transform model assumes that $\mathbf{W}\mathbf{p} \approx \mathbf{z}$, where $\mathbf{W}  \in \mathbb{R}^{K \times n}$ denotes a sparsifying transform and $\mathbf{z}$ is assumed to emit a sparse structure (where zeros correspond to the transform rows that approximately annihilate the signal). The sparsifying transform model is a generalization \cite{sabres} of the analysis model \cite{namel} that assumes that applying $\mathbf{W}$ to a signal produces several zeros in the output. These models can be viewed as a union of null-spaces model for signals.\footnote{Depending on the signal set, either a compact (i.e., without too many atoms) dictionary or sparsifying transform may be best suited for them.} For the transform model, sparse transform-domain approximations are obtained exactly by simple (e.g., hard or soft) thresholding \cite{sabres}. 

The learning of dictionaries and sparsifying transforms from a collection of signals has been explored in many recent works \cite{ols, elad, Yagh, smith1, aksvd, sabres}. The learning problems are often highly non-convex (e.g., nonconvexity due to product of matrices structure, or nonconvex constraints such as using the $\ell_{0}$ ``norm" ), and many learning algorithms lack proven convergence guarantees or model recovery guarantees. Recent works \cite{spel2b, agra1, arora1, yint3, bao1, bao2, agra2, sbclsTS2} have studied the convergence of specific learning algorithms. Some of these works demonstrate promising results in applications for efficient synthesis dictionary \cite{bao1, bao2, yint3} or transform \cite{sbclsTS2, bao22} learning algorithms and prove convergence of the learning methods to the critical points (or generalized stationary points \cite{vari1}) of the underlying costs. These works all employ the $\ell_0$ ``norm" or other nonconvex regularizers in their costs, which work well in applications. Other works such as \cite{agra1, agra2} use the $\ell_1$ norm and prove the recovery of the underlying generative model for specific learning methods using an alternating minimization approach, but rely on restrictive assumptions on sparsity and the initial error. Arora et al.~\cite{arora2015simple} analyzed alternating minimization approaches to synthesis dictionary learning and provided a convergence radius $O(1/\log n)$ (i.e., initializations within the radius provide convergence), but the upper bound on the iterate error included a non-zero offset and fresh samples may be needed in each iteration. In~\cite{arora1}, the authors propose and analyze polynomial time algorithms for learning  overcomplete dictionaries but comment that their algorithms are not suitable for large-scale applications due to computational runtime costs. Moreover, these and other schemes \cite{spel2b} have not been demonstrated to be practically powerful in applications such as inverse problems and can be computationally expensive.

Often, additional properties may be enforced on the model during learning such as incoherence \cite{barchi1, irami}, nonsingularity \cite{spel2b}, etc. In a recent two-part work, Sun et al. \cite{sun1, sun2} focused on complete dictionaries and studied the geometric properties of the non-convex objective for dictionary learning over a high dimensional sphere. Their work showed with high probability that there were no spurious local minimizers and proposed an algorithm that converged to local minimizers. While other works such as \cite{arora2015simple, schnass2018convergence, chatterji2017alternating,bai2018subgradient} provided theoretical guarantees for specific dictionary learning algorithms, they do not enforce structural constraints on the dictionary during learning. This work enforces the learned model to be \emph{unitary}, which has been demonstrated to be both effective and computationally advantageous in practice \cite{sravTCI1, bao2013fast, sabres3, hanif2014maximum}. While alternating minimization algorithms for general synthesis dictionary learning typically require iterative or greedy or other approximate techniques to solve the subproblems \cite{elad, agra1}, the corresponding algorithms with unitary models, even with the $\ell_0$ ``norm", typically have efficient closed-form solutions~\cite{sravTCI1}. Although unitary dictionary learning has shown promise empirically, there has been a lack of theoretical guarantees for proposed methods \cite{bao2013fast, hanif2014maximum, sabres3}. Given the recently increasing interest in such models and their effectiveness in applications such as inverse problems \cite{sravTCI1, sravwohl}, our work focuses on analyzing the convergence of algorithms for such structured nonconvex learning problems.

In the following section, we outline the \emph{structured} (unitary) operator learning approach that involves simple, computationally cheap updates. We investigate its convergence properties in the rest of the paper.

\subsection{Unitary Operator Learning Formulation and Algorithm}
Given an $n \times N$ training data set $\bP$, whose columns represent training signals, our goal is to find an $n \times n$ sparsifying transformation matrix $\bW$ and an $n \times N$ sparse coefficients (representation) matrix $\Z$ by solving the following constrained optimization problem:
\begin{equation}
\argmin_{\bW, \Z} \|\bW\bP - \Z \|^2_F  \;\;\; \text{s.t.} \;\; \bW^T\bW = \Id, \, \begin{Vmatrix}
\Z_{(.,j)}
\end{Vmatrix}_{0}\leq s \,\,\, \forall j.
\label{eq:opt_program}
\end{equation}
We focus on the learning of unitary sparsifying operators ($\bW^T\bW = \Id$ with $\Id$ denoting the identity matrix) that have shown promise in applications such as denoising \cite{sabres3} and medical image reconstruction \cite{sravTCI1} compared to other models. The columns $\Z_{(.,j)}$ of $\Z$ have at most $s$ non-zeros (measured using the $\ell_{0}$ ``norm"), where $s$ is a given parameter. Alternatives to Problem~\eqref{eq:opt_program} involve replacing the column-wise sparsity constraints with a constraint on the total sparsity (aggregate sparsity) of the entire matrix $\Z$, or using a sparsity penalty (e.g., $\ell_{p}$ penalties with $0 \leq p \leq  1$).

Problem~\eqref{eq:opt_program} is an instance of sparsifying transform learning \cite{pfisbres, sabres}, with unitary constraint on the operator or filter set. Sparsifying transform learning generalizes conventional analysis dictionary learning. Analysis dictionary learning approaches typically minimize the $\ell_0$ or $\ell_1$ norm of $\mathbf{WP}$ subject to non-triviality constraints on $\mathbf{W}$ that prevent trivial solutions such as the all-zero matrix \cite{nam}. Popular variations to model noisy data minimize $\left \| \mathbf{P}-\mathbf{\hat{P}} \right \|_{F}^2$ subject to sparsity-type constraints on $\mathbf{W\hat{P}}$ and constraints on $\mathbf{W}$ \cite{yana,rumi,aksvd}. Sparse coding in the latter variation (i.e., estimating $\mathbf{\hat{P}}$ for fixed $\mathbf{W}$) can be NP-hard in general. Problem (1.1) learns a different generalization of the analysis model, where $\mathbf{P}$ is assumed ``approximately" sparse in the \emph{transformed domain}. Natural signals and images are well-known to be approximately sparse in the wavelet or discrete cosine transform (DCT) domain, etc., and such sparsifying transforms have also been exploited for denoising data. Problem~\eqref{eq:opt_program} with the unitary constraint on $\bW$ is also equivalent to learning a synthesis dictionary $\bW^{T}$ for sparsely approximating the training data $\bP$ as $\bW^{T}\Z$.

Alternating minimization algorithms are commonly used for learning synthesis dictionaries \cite{eng, elad, zibul, smith1}, analysis dictionaries (such as the noisy data variation above) \cite{yana,aksvd}, and sparsifying transforms \cite{pfisbres, sabres, saiwen}. In particular, unlike sparse coding in the first two models, which can be NP-hard in general, computing sparse approximations in the transform model is cheap involving thresholding, and thus, various efficient and effective algorithms have been proposed for transform learning with different properties or constraints on $\mathbf{W}$. One could alternate between solving for $\bW$ and $\Z$ in Problem~\eqref{eq:opt_program} \cite{sabres3, sbclsTS2}. In this case, the solution for the $t$th $\Z$ update (\emph{sparse coding step}) is obtained as $\Z_{(.,j)}^{t} = H_{s}(\bW^{t-1} \bP_{(.,j)})$ $\forall j$, where $\bP_{(.,j)}$ and $\Z_{(.,j)}^{t}$ denote the $j$th columns of $\bP$ and $\Z^{t}$ respectively, and the operator $H_{s}(\cdot)$ zeros out all but the $s$ largest magnitude elements of a vector, leaving other entries unchanged (i.e., thresholding to $s$ largest magnitude elements). The solution for the subsequent $\bW$ update (\emph{operator update step}) is obtained by first computing the full singular value decomposition (SVD) of $\Z^t \bP^{T}$ as $\V \mathbf{\Sigma} \U^{T}$, and then $\bW^{t}=\V \U^{T}$. The algorithm repeats these relatively cheap updates until convergence. The overall method is provided in Algorithm~\ref{alg:alternate}.

Although Problem~\eqref{eq:opt_program} is nonconvex because of the $\ell_{0}$ sparsity and unitary operator constraints, the alternating minimization algorithm involves cheap and closed-form update steps. The thresholding-type solution for the sparse coding step readily generalizes to alternative formulations such as with an $\ell_0$ aggregate sparsity constraint or sparsity penalties \cite{sbclsTS2}. These advantages of unitary operator learning (that also extend to general sparsifying transform learning \cite{saiwen}) and its effectiveness in applications \cite{sravTCI1} render it quite attractive vis-a-vis alternatives such as overcomplete synthesis dictionary learning, and hence we investigate it further in this work.

Problem~\eqref{eq:opt_program} is also interpreted as training an efficient convolutional or filterbank model \cite{pfisbres, sravwohl} for 2D (or higher dimensional) images, with thresholding-type nonlinearities. To see this, we observe that if overlapping patches of an image or collection of images of size $\sqrt{n} \times \sqrt{n}$ are (vectorized and) used for training with periodic image boundary condition (so patches at image boundaries wrap around on the opposite side of the image) and a patch stride of $1$ pixel in the horizontal and vertical directions (maximal patch overlap), then the transform learned by Problem~\eqref{eq:opt_program} is applied to sparse code the data by first applying each row to all the image patches via inner products, followed by thresholding operations. The sparse outputs of the transform are thus generated by circularly convolving its reshaped (into 2D patches and flipped) rows with the image followed by thresholding. Thus, Problem~\eqref{eq:opt_program} adapts a collection of orthogonal sparsifying filters for images, and Algorithm~\ref{alg:alternate} can also be implemented with filtering-based operations.

\subsection{Contributions}
In this work, we focus on investigating the convergence properties of the aforementioned efficient alternating minimization algorithm for unitary sparsifying operator learning. Recent works have shown convergence of the algorithm (or its variants) to critical points of the equivalent unconstrained problem \cite{sbclsTS2, bao22, sravTCI1}, where the constraints are replaced with barrier penalties (that take value $+ \infty$ when the constraint is violated and $0$ otherwise). Here, we further prove the fast local linear convergence of the algorithm to the underlying data models. Our results hold under mild assumptions that depend on the properties of the underlying sparse coefficients matrix $\Z$. In addition to showing convergence, we also characterize the convergence radius and rate and discuss general and example distributions of the data for which our results hold. We also show experimentally that our assumptions and convergence guarantees hold for well-known probabilistic models of $\Z$. Our experiments and initial arguments indicate that the learning algorithm is robust to initialization.

\subsection{Organization}
The rest of this paper is organized as follows. Section~\ref{sec2a} presents the main convergence results and proofs. Section~\ref{sec3} presents experimental results supporting the statements in Section~\ref{sec2a} and illustrating the empirical behavior of the transform learning algorithm. In Section~\ref{sec4}, we conclude with proposals for future work.

\begin{algorithm}
 \caption{Alternating Optimization for \eqref{eq:opt_program}} \label{alg:alternate}
\begin{algorithmic}
		\State \textbf{Input: } Training data matrix $\bP$, maximum iteration count $L$, sparsity $s$
		\State \textbf{Output: }$\bW^L$, $\Z^L$ 
		\State \textbf{Initialize: } $\bW^0$ and $t=1$
	   \For{$t \leq L$}
		   \State $\Z_{(.,j)}^t = H_{s}\begin{pmatrix}\bW^{t-1} \bP_{(.,j)} \end{pmatrix}$ $\forall$ $j$ \Comment{$\Z^t = \argmin_{\Z: \|
\Z_{(.,j)} \|_{0}\leq s \,\, \forall j} \|\bW^{t-1}\bP - \Z\|^2_F $} 
		  \State $\bP {\Z^t}^{T}$=$\U^t \bSigma^t {\V^t}^T$  \Comment{$\bW^t = \argmin_{\bW: \bW^T\bW = \Id} \|\bW \bP - \Z^{t} \|^2_F $}  
		  \State $\bW^t = \V^t{\U^t}^T$
		  \State $t = t+1$
	   \EndFor

\end{algorithmic}
\end{algorithm}

\section{Convergence Analysis} \label{sec2a}
The main contribution of this work is the convergence analysis of Algorithm~\ref{alg:alternate}. We begin this section outlining notation and the assumptions under which our analysis operates. Following this, we summarize the theoretical guarantees of our work and present the proofs for these results.

\subsection{Notation} 
We adopt the following notation in the rest of the paper. Matrix $\Z$ denotes the $n \times N$ sparse coefficients matrix, $\bW$ is the $n \times n$ sparsifying transform, and $\bP$ denotes the $n \times N$ data set. The $t$th approximation of a variable (iterate in the algorithm) is denoted $(\cdot)^t$. The capital letter $T$ is reserved for the transpose operator, i.e., the variable ${\Z^t}^T$ should be read as the transpose the of $t$th approximation for $\Z$. With the exception of $T$, capitalized letters are used for matrices and lowercase letters are used for vectors, with further subscripts denoting the row, column, or entry of the matrix or vector. The $i$th row, $j$th column, and the $(i,j)$th entry of a matrix $\M$ are denoted $\row{\M}$, $\col{\M}$, and $\entry{\M}$, respectively. For any vector $\mathbf{v}$, $S(\mathbf{v})$ denotes the function that returns the support, i.e., $S(\mathbf{v}) := \{i: \mathbf{v}_i \neq 0 \}$, where $\mathbf{v}_i$ denotes the $i$th entry (scalar) of $\textbf{v}$. The operator $H_{s}(\mathbf{v})$ leaves the $s$ largest magnitude elements of $\mathbf{v}$ unchanged and zeros out all other entries (i.e., thresholding to $s$ largest magnitude elements). Matrix $\D_k$ denotes an $n \times n$ diagonal matrix of ones and a zero at location $(k,k)$. Additionally, $\tilde{\D}_k$ denotes an $N \times N$ diagonal matrix that has ones at entries $(i,i)$ for $i \in S(\Zopt_{(k,\cdot)})$ and zeros elsewhere, and matrix $\Zopt$ is defined in Section \ref{sec:assump} (see assumption $(A_1)$). The Frobenius norm, denoted $\| \X \|_F$, is the square root of the sum of squared elements of $\X$ and $\| \X \|_2$ denotes the spectral norm. Lastly, $\Id$ denotes the appropriately sized identity matrix.

\subsection{Assumptions} \label{sec:assump}
 We begin with the following assumptions that will be used in various results:
\begin{enumerate}[label=(\subscript{A}{\arabic*})]
\itemsep0em 
\item \textbf{Generative model:} There exists a $\Z^*$ and unitary $\bW^*$ such that  $\bW^* \bP = \Z^*$ and $\left \| \bP \right \|_{2}=1$ (normalized data). 
\item \textbf{Sparsity:} The columns of $\Z^*$ are $s$-sparse, i.e., $\| \col{\Z}^* \|_0 = s$ $\forall$ $j$. 
\item \textbf{Spectral property:} The underlying $\Zopt$ satisfies the bound $\kappa^{4}\begin{pmatrix} 
\Zopt
\end{pmatrix}\max_{1\leq k \leq n} \| \D_k \Zopt \tilde{\D}_k \|_2 < 1$, where $\kappa(\cdot)$ denotes the condition number (ratio of largest to smallest singular value).
\item \textbf{Orthogonal coefficients:} The rows of $\Z^*$ are orthonormal, i.e., $\Zopt\Zopt^T = \Id$.
\item \textbf{Initialization:} $\| \bW^0 - \bW^* \|_{F} \leq \epsilon$ for an appropriate small $\epsilon >0$. \end{enumerate}

The first two assumptions are on the model for the data, i.e., we would like the algorithm to find an underlying (unitary) sparsifying transform and representation matrix such that $\bW^* \bP =\Z^*$ holds (data generated as $\bW^{*^T}\Z^*$), where the columns of $\Z^*$ have $s$ nonzeros. The coefficients are assumed ``structured" in assumption $(A_3)$, satisfying a spectral property, which will be used to establish our theoretical results. When $s=2$ and $\Zopt\Zopt^T = \Id$, we show that assumption $(A_3)$ simplifies to very intuitive and deterministic conditions of uniqueness (the support of no two rows of $\Zopt$ fully coincide) and irreducibility (each row of $\Zopt$ has at least one nonzero, i.e., each atom or row of $\bW^*$ contributes to at least one nonzero in the data representation). More generally or when $s>2$, the condition that each row of $\Zopt$ has at least one nonzero is still required in order for $(A_3)$ to hold (as otherwise $\kappa(\Zopt) = \infty$) but the assumption does not reduce to a simple setting. We will present an analysis and empirical results showing that the spectral property holds for well-known probabilistic models. The analysis will also show that in general the underlying matrices $\Zopt$ and $\D_k \Zopt \tilde{\D}_k$ defining the spectral property behave similarly for the probabilistic models as $N \to \infty$ as for the special $s=2$ case above. Assumption $(A_4)$ on orthogonality of coefficient matrix (normalized) rows simplifies the condition in assumption $(A_3)$ (since $\kappa(\Zopt) =1$) and is used in presenting/proving one version of the results, but is omitted in the generalization. For well-known probabilistic models of the coefficient matrix, we will show that the orthogonality holds asymptotically. Assumption $(A_5)$ on algorithm initialization states that the initial sparsifying transform, $\bW^0$ is sufficiently close to the solution $\bW^*$. Such an assumption has also been made in other works, where the issue of good initialization is tackled separately~\cite{agra1, agra2}. Section~\ref{sec2main} characterizes $\epsilon$ in assumption $(A_5)$ in more detail. While the main results in Section~\ref{sec2maina} use assumption $(A_5)$, we also discuss a generalization in Section~\ref{sec2maind}. Our theoretical results are stated next.

\subsection{Results}\label{sec2main}
In the following, Theorem~\ref{thm:main1} first presents a convergence result using all the aforementioned assumptions. Then Theorem~\ref{thm:main2} generalizes the result by dropping Assumption $(A_4)$. Proposition~\ref{conj} states that Assumption $(A_3)$ holds under a general probabilistic model on the sparse representation matrix $\Zopt$. We also later show numerical results illustrating Proposition~\ref{conj}. We also provide a corollary on a special case of Theorems~\ref{thm:main1} and \ref{thm:main2} and some remarks. In particular, Remark~\ref{rm11} discusses dropping the data normalization assumption in $(A_1)$, and Remark~\ref{rm12} discusses the effect of noise on Theorems~\ref{thm:main1} and \ref{thm:main2}. Proposition~\ref{convradius} and Remark~\ref{rm1b} characterize and discuss the behavior of $\epsilon$ in Assumption $(A_5)$.

\subsubsection{Main Results} \label{sec2maina}
\begin{theorem} Under Assumptions $(A_1)$-$(A_5)$, the Frobenius error between the iterates generated by Algorithm~\ref{alg:alternate} and the underlying model in Assumptions $(A_1)$ and $(A_2)$ is bounded as follows:
\begin{equation}
\| \Z^t - \Zopt \|_F  \leq  q^{t-1} \epsilon, \;\; \| \bW^t - \Wopt \|_F  \leq  q^t \epsilon, 
\end{equation}
where $q \doteq \max_{1 \leq k \leq n} \| \D_k \Zopt \tilde{\D}_k \|_2 < 1$ and $\epsilon$ is fixed based on the initialization.
\label{thm:main1}
\end{theorem}
Here, the symbol ``$\doteq$" indicates equality up to first-order terms, with the other terms negligible. We will mostly only refer to the dominant component of $q$ in the discussions. The latter components are considered in more detail in the convergence radius analysis later (Section~\ref{convradiussection} and Appendix~\ref{app1}).

\begin{theorem} Under Assumptions $(A_1)$-$(A_3)$ and $(A_5)$, the iterates in Algorithm~\ref{alg:alternate} converge linearly to the underlying model in Assumptions $(A_1)$ and $(A_2)$, i.e., the Frobenius error between the iterates and the underlying model satisfies
\begin{equation}
\| \Z^t - \Zopt \|_F  \leq  q^{t-1} \epsilon, \;\; \| \bW^t - \Wopt \|_F  \leq  q^t \epsilon,  \label{thm2eq}
\end{equation}
where $q \doteq \kappa^{4}\begin{pmatrix}
\Zopt
\end{pmatrix}\max_{1\leq k \leq n} \| \D_k \Zopt \tilde{\D}_k \|_2 < 1$ and $\epsilon$ is fixed based on the initialization. 
\label{thm:main2}
\end{theorem}

Next, we discuss special cases of Theorem~\ref{thm:main1} and Theorem~\ref{thm:main2} when $s=2$. In the case of Theorem~\ref{thm:main1}, a simple intuitive condition that the supports of no two rows of $\Zopt$ fully overlap ensures linear convergence ($q<1$), i.e., ensures Assumption $(A_3)$ holds.

\begin{corollary}{\textbf (Case $s=2$)} For Theorem~\ref{thm:main1}, when $s=2$ and no two rows of $\Zopt$ have identical support, then $q<1$ holds in Assumption $(A_3)$.
For Theorem~\ref{thm:main2} (without Assumption ($A_4$)), when $s=2$, then $q<1$ holds in Assumption ($A_3$) if  $\| \row{\Zopt} |_{S(
\Zopt_{(i,.)}) \cap S(
\Zopt_{(k,.)})} \|_{2} < \kappa^{-4}(\Zopt)$ for all $i \neq k$,
where the norm is computed only with respect to the elements of $\row{\Zopt}$ in the support $S(\Zopt_{(i,.)}) \cap S(\Zopt_{(k,.)})$.
 \label{coro1}
\end{corollary}

Remark~\ref{rm11} discusses the effect of dropping the data normalization assumption (in $(A_1)$) on the convergence rate. In particular, the convergence rate factor $q$ is modified by being normalized by $\| \bP \|_{2} = \| \Zopt \|_{2}$, keeping it invariant to scaling of $\Zopt$.
\begin{remark}
When the unit spectral norm condition on $\bP$ in Assumption $(A_1)$ is dropped, the $\| \bW^t - \Wopt \|_F$ bound in Theorem~\ref{thm:main2} holds with $q \doteq \begin{pmatrix}
\kappa^{4}\begin{pmatrix}
\Zopt
\end{pmatrix}/\| \bP \|_{2}
\end{pmatrix}\max_{1\leq k \leq n} \| \D_k \Zopt \tilde{\D}_k \|_2$. The bound $\| \Z^t - \Zopt \|_F \leq q^{t-1} \epsilon$ as in \eqref{thm2eq} holds with the aforementioned $q$ but with $\epsilon$ replaced by $\| \bP \|_{2} \epsilon$. \label{rm11}
\end{remark}

As will be clear from the proofs in Section~\ref{secproof}, when Assumption $(A_2)$ stating $\| \col{\Z}^* \|_0 = s$ is relaxed to $\| \col{\Z}^* \|_0 \leq s$, then the (common) linear contraction factor $q$ for the error in each iteration in Theorem~\ref{thm:main2} (with respect to previous iteration's error) is replaced with $q(t) \doteq \kappa^{4}
\begin{pmatrix}
\Zopt
\end{pmatrix} \max_{1 \leq k \leq n} \| \D_k \Zopt \tilde{\D}_k^{t} \|_2$,
where $\tilde{\D}_k^{t}$ is defined similar to $\tilde{\D}_k$ but with respect to $S(\Z^{t}_{(k,\cdot)})$ (which is shown in Section~\ref{secproof} to contain $S(\Zopt_{(k,\cdot)})$).

Finally, we have the following generalization of Theorem~\ref{thm:main2} for noisy models. The $\epsilon$ in Assumption $(A_5)$ would be smaller in the presence of noise and the noise is assumed small enough so that the support recovery and Taylor Series convergence properties used in the proofs in Section~\ref{secproof} hold.

\begin{remark}
When a noisy model of the data is used in Assumption $(A_1)$, i.e., $\bW^* \bP = \Z^* + \mathbf{H}$, where $\mathbf{H}$ denotes noise, then for sufficiently small noise, Theorem~\ref{thm:main2} holds, except that the term $C \| \mathbf{H} \|_F$, where $C>0$ is a constant, is added to the right hand side of \eqref{thm2eq}. \label{rm12}
\end{remark}

\subsubsection{Convergence Rate} \label{sec2mainb}
While our main results assume that the spectral property in Assumption $(A_3)$ holds, the next result discusses the scenario and models under which the assumption $q<1$ is generally valid.
\begin{proposition} Suppose the locations of the $s$ non-zeros in each column of $\Zopt$ are chosen independently and uniformly at random, and the non-zero entries are i.i.d. with mean zero and variance $n/sN$. Then, for fixed $s$, $n$, and $s<n$, we have that  $q_{N} \triangleq \begin{pmatrix}
\kappa^{4}\begin{pmatrix}
\Zopt
\end{pmatrix}/\| \bP \|_{2}
\end{pmatrix}\max_{1\leq k \leq n} \| \D_k \Zopt \tilde{\D}_k \|_2$ $ < 1$ for large enough $N$ with high probability. In particular, we have the following limit almost surely:
\begin{equation}
q^{*}=\lim_{N\to \infty} q_{N} = \sqrt{\frac{s-1}{n-1}}.
\label{qlim1}
\end{equation}
\label{conj}
\end{proposition}
Proposition~\ref{conj} holds for several well-known distributions of $\Zopt$ such as when its column supports are drawn independently and uniformly at random and the nonzero entries are a) i.i.d. with $\entry{\Zopt} \sim \mathcal{N}(0,\frac{n}{sN})$ or b) i.i.d. scaled (by $\sqrt{n/sN}$) random signs with ``+" and ``-" being equally probable. Section~\ref{sec3} empirically shows the algorithm's convergence and the behavior of $q$ when $s=O(n)$, a commonly used sparsity criterion in many applications (i.e., with $s= \alpha n$, where $\alpha<1$ is a small fraction).

\subsubsection{Convergence Radius} \label{convradiussection}
While the main convergence results make use of Assumption $(A_5)$, here we discuss the behavior of the convergence radius $\epsilon$, including when the number of training signals $N \to \infty$. The following proposition and remark characterize a sufficient $\epsilon$ in Assumption $(A_5)$ for Theorem~\ref{thm:main1} and Theorem~\ref{thm:main2}.
\begin{proposition} 
The iterate convergence in Theorem~\ref{thm:main1} holds when the radius of convergence $\epsilon$ in Assumption~$(A_5)$ satisfies $\epsilon < \min(\epsilon_1,\epsilon_2)$, where $\epsilon_1 = 0.5 \min_{1\leq j \leq N} \beta \left( \col{\Z}^*/\|\col{\Z}^*\|_{2} \right)$ with $\beta(\cdot)$ computing the smallest nonzero magnitude in a vector, and $\epsilon_2 = \max_{0 \leq \epsilon_0 < \sqrt{2}-1}f(\epsilon_0)$ with $f(\epsilon_0) \triangleq \min \begin{pmatrix}
\frac{1 - \max_{1\leq k \leq n} \| \D_k \Zopt \tilde{\D}_k \|_2}{C(\epsilon_0)},\epsilon_0
\end{pmatrix}$, and $C(\epsilon_0)$ is defined as follows:
\begin{equation}
C(\epsilon_0)= 2 + \frac{9 \sqrt{2}}{8(1-2\epsilon_0-\epsilon_0^2)^{1.5}}+\frac{\sqrt{2}}{1-\epsilon_0}.    
\label{constant}    
\end{equation}
\label{convradius}
\end{proposition}
In Proposition~\ref{convradius}, $\epsilon_1$ arises from the sparse coding step of Algorithm~\ref{alg:alternate} and ensures recovery of the support of the underlying sparse coefficients. The bound $\epsilon_2$ arises in the operator update step of Algorithm~\ref{alg:alternate} and is primarily to ensure the convergence and boundedness of Taylor Series expansions discussed in the proof. The largest permissible $\epsilon_2$ that suffices for Theorem~\ref{thm:main1} is obtained by maximizing the function $f(\epsilon_0)$ over $0 \leq \epsilon_0 < \sqrt{2}-1$. The end points of this interval both correspond to $f(\epsilon_0)=0$. So the maximum of the continuous (nonnegative) function $f(\epsilon_0)$ would occur inside the interval. The constant $C(\epsilon_0)$ is monotone decreasing as $\epsilon_0 \to 0$ with the limiting $C(0)=5.01$. The result indicates that the radius of convergence depends on the properties of the underlying sparse coefficients.

\begin{remark}
Proposition~\ref{convradius} holds for Theorem~\ref{thm:main2} but with $\epsilon_2$ depending on $\kappa(\Zopt)$. In particular, as $\kappa(\Zopt) \to 1$, $\epsilon_2$ takes the same form as in Proposition~\ref{convradius}. Moreover, for the distributions in Proposition~\ref{conj}, $\kappa(\Zopt) \to 1$ and $\epsilon_2 \to \max_{0 \leq \epsilon_0 < \sqrt{2}-1} \min \begin{pmatrix}
\frac{1 - \sqrt{\frac{s-1}{n-1}}}{C(\epsilon_0)},\epsilon_0
\end{pmatrix}$ almost surely as $N \to \infty$.
 \label{rm1b}
\end{remark}
Remark~\ref{rm1b} indicates that $\epsilon_1$ in Proposition~\ref{convradius} remains unchanged for Theorem~\ref{thm:main2}. However, the $\epsilon_2$ arising from the operator update step depends on $\kappa(\Zopt)$. For example, a bound of $\kappa^{-2}(\Zopt)$ (smaller for larger condition numbers) ensures the convergence of one of the Taylor Series in the proof in Section~\ref{thmmain2proof}. Importantly, for the distributions of $\Zopt$ in Proposition~\ref{conj}, the limiting value of $\epsilon_2$ stated in Remark~\ref{rm1b} depends only on the ratio $(s-1)/(n-1)$.

The limiting behavior of $\epsilon_1 = 0.5 \min_{1\leq j \leq N} \beta \left( \col{\Z}^*/\|\col{\Z}^*\|_{2} \right)$ as $N \to \infty$ would depend on the distribution of $\Zopt$. Appendix~\ref{app2} discusses some example distributions that satisfy the assumptions in Proposition~\ref{conj} and have the nonzero values bounded away from zero, for which $\epsilon_1 \geq C/\sqrt{s}$ holds for each $N$, where $C$ is a positive constant. Practically, peak physical intensity and numerical precision bound nonzero entries of the sparse coefficient matrix. In practice, we expect the radius $\epsilon$ in Proposition~\ref{convradius} to be limited more by $\epsilon_1$, since $\epsilon_2$ depends approximately only on the ratio $s/n$ for large $N$ (Remark~\ref{rm1b})
and would be a constant when $s \propto n$.

\subsubsection{Discussion of Generalization of Convergence Radius Assumptions} \label{sec2maind}
Here, we discuss the effect of  $\epsilon$ values larger than in Proposition~\ref{convradius} (or Remark~\ref{rm1b}) on the convergence of Algorithm~\ref{alg:alternate}.
The following lemma shows the behavior of the sparse coding error for general algorithm initializations (or general $\epsilon$ values that may not ensure support recovery).
\begin{lemma}
For $t=1$ in Algorithm~\ref{alg:alternate} and under Assumptions $(A_1)$ and $(A_2)$ and denoting $\E^{0} = \bW^{0} - \bW^*$ with $\left \| \E^0 \right \|_{F} \leq \epsilon$ for some non-negative $\epsilon$, we have that 
\begin{equation}
\| \Z^1 - \Zopt \|_F \leq 2\begin{Vmatrix}
\E^0
\end{Vmatrix}_{F} \leq 2 \epsilon.
\label{errorsparse}
\end{equation}
\label{sparsecodegeneral}
\end{lemma}
Appendix~\ref{app3} provides the proof of Lemma~\ref{sparsecodegeneral}. Lemma~\ref{sparsecodegeneral} suggests that regardless of how close the initial transform is to the underlying model, the bound on the sparse coding error is at most twice that in Theorem~\ref{thm:main2}.
In this case, the contraction factor in the operator update step would need to satisfy $q<0.5$ in order to consistently decrease the error.
We have $q(t) \, \doteq \, \kappa^{4}\begin{pmatrix}
\Zopt
\end{pmatrix} \max_{1 \leq k \leq n} \| \D_k \Zopt \tilde{\D}_k^{t} \|_2$ for the operator update step, where $\tilde{\D}_k^t$ is a diagonal matrix with ones at entries $(i,i)$ for $i \in S(\Z^{t}_{(k,\cdot)} - \Zopt_{(k,\cdot)})$ and zeros elsewhere. 
If the supports of $\Z^{t}_{(k,\cdot)}$ and $\Zopt_{(k,\cdot)}$ are mismatched, then $\tilde{\D}_k^t$ could in general have more ones than $\tilde{\D}_k$. In other words, $q(t)$ could be larger than the $q$ in Theorem~\ref{thm:main2}.
Thus, the (larger) effective (or overall) factor of $2 q(t)$ could lead to slow convergence initially from more general initializations.
This is also corroborated by the experiments in Section~\ref{sec3}, where slower convergence is observed from general initializations until the underlying support is fully recovered, at which point, the linear convergence behavior predicted in Theorem~\ref{thm:main2} is fully observed, with a similar rate of convergence regardless of initialization.

\subsection{Proofs of Theorems, Corollary, Propositions, and Remarks} \label{secproof}
We first prove Theorem~\ref{thm:main1} and then the proof of Theorem~\ref{thm:main2} is briefly presented highlighting the distinctions arising from the generalization. The proof of Corollary~\ref{coro1} is presented for the case of Theorem~\ref{thm:main1} (the proof for the case of Theorem~\ref{thm:main2} is similar). The proof of Remark~\ref{rm12} follows along the same lines as those of the theorems, and is omitted. Finally, the proof of Proposition~\ref{conj} is presented. The proofs of Proposition~\ref{convradius} and Remark~\ref{rm1b} are outlined in Appendix~\ref{app1}.

To prove Theorem~\ref{thm:main1}, we will first prove two supporting lemmas that establish properties of the iterates. First, Lemma~\ref{lem:z} shows that the error between the iterate $\Z^1$ and $\Zopt$ is bounded and the bound depends on the approximation error with respect to $\Wopt$ for the initial $\bW^{0}$ (bounded by $\epsilon$ as in Assumption $(A_5)$). Lemma~\ref{lem:w} and Lemma~\ref{lem:z_not_orth} show that the error between the first $\bW$ iterate ($\bW^{1}$) and $\Wopt$ is bounded above by $q \epsilon$ for Theorem~\ref{thm:main1} and Theorem~\ref{thm:main2}, respectively. Similar bounds are shown to hold for subsequent iterations. Therefore, for Algorithm~\ref{alg:alternate} to converge linearly, one only needs $q<1$ as in Assumption $(A_3)$ or as established by Proposition~\ref{conj}. The scaling indicated in Remark~\ref{rm11} follows from the proofs of Lemmas~\ref{lem:z} and \ref{lem:z_not_orth}.

\subsubsection{Proof of Theorem~\ref{thm:main1}}
For our proofs, we define the sequences $\left \{ \E^t \right \}$ and $\left \{ \mathbf{\Delta}^t \right \}$ such that 
\begin{align}
\bW^t &= \Wopt + \E^t, \label{eq:deltaw}\\
\Z^t &= \Zopt + \mathbf{\Delta}^t. \label{eq:deltaz}
\end{align}

\begin{lemma}{\textbf (Approximation error for $\Z$)} For $t=1$ in Algorithm~\ref{alg:alternate} and under Assumptions $(A_1) - (A_5)$, the Frobenius norm of the approximation error of the estimated sparse coefficients with respect to $\Zopt$ is bounded by $\epsilon$ as defined in $(A_5)$. In particular, we have that 
$$\| \Z^1 - \Zopt \|_F \leq \| \E^{0}\|_F , $$
where $\| \E^{0} \|_{F} \leq \epsilon$.
\label{lem:z}
\end{lemma}
\begin{proof}
For each column indexed by $j = 1, ..., N$, of the sparse coefficients matrix $\Z^{1}$, the following hold:
\begin{align}
\col{\Z}^1 \stackrel{}{=} H_s( & \bW^{0}  \col{\bP}) \stackrel{(Eq. \ref{eq:deltaw})}{=} H_s(\Wopt \col{\bP} + \E^{0} \col{\bP}) \nonumber  \\
&\stackrel{(A_1)}{=} H_s(\col{\Z}^* + \E^{0} \col{\bP}) \nonumber \\ 
&\stackrel{(A_5)}{=} \col{\Z}^* + \mathbf{\Gamma}_j^1 \E^{0} \col{\bP}, \label{eq:zi}
\end{align}
where $\mathbf{\Gamma}_j^1$ is a diagonal matrix with a one in the $(i,i)$th entry if $i \in S(\col{\Z}^1)$ and zero otherwise and $\E^{0}$ is as defined in \eqref{eq:deltaw}.
The last equality above follows from the fact that the support of $\col{\Z}^1$ includes that of $\col{\Z}^*$,
for small enough $\epsilon$ (assumption $(A_5)$). In particular, since $\begin{Vmatrix}
\col{\bP}
\end{Vmatrix}_{2} = \begin{Vmatrix}
\col{\Z}^*
\end{Vmatrix}_{2}$, we have 
\begin{align*}
\begin{Vmatrix}
\E^{0} \col{\bP}
\end{Vmatrix}_{\infty}\leq \begin{Vmatrix}
\E^{0} \col{\bP}
\end{Vmatrix}_{2} \leq \begin{Vmatrix}
\E^{0}
\end{Vmatrix}_{F} \begin{Vmatrix}
\col{\Z}^*
\end{Vmatrix}_{2}.
\end{align*}
Therefore, whenever $\| \E^0 \|_F \leq \epsilon < \frac{1}{2} \min_j \beta \left( \frac{\col{\Z}^*}{\|\col{\Z}^*\|_{2}} \right)$ with $\beta(\cdot)$ being the smallest nonzero magnitude vector entry, the support of $\col{\Z}^1$ includes\footnote{In this case, the support of $\col{\Z}^1$ in fact coincides with that of $\col{\Z}^*$. If we relaxed Assumption $(A_2)$ from $\| \col{\Z}^* \|_0 = s$ to $\| \col{\Z}^* \|_0 \leq s$ $\forall$ $j$, then $S\left ( \col{\Z}^* \right )\subseteq S\left ( \col{\Z}^1 \right )$ holds, and the lemma still holds.}
that of $\col{\Z}^*$ (the entries of the perturbation $\E^{0} \col{\bP}$ are not large enough to change the support).
The following results then hold:
\begin{align*}
\|\Z^1 - \Zopt \|_F^{2} &\stackrel{(Eq. \ref{eq:zi})}{=} \|[\mathbf{\Gamma}_1^1 \E^{0} \bP_{(\cdot,1)} , ... , \mathbf{\Gamma}_N^1 \E^{0} \bP_{(\cdot,N)}] \|^2_F \\
& \stackrel{(i)}{\leq} \| \E^{0} \bP\|_F^{2} \stackrel{(ii)}{\leq} \| \E^{0}\|_F^2 \| \bP \|^2_2 \stackrel{(A_1)}{=} \| \E^{0}\|_F^2.
\end{align*}
Here, $(i)$ follows by definition of $\mathbf{\Gamma}_j^1$; step $(ii)$ holds for the Frobenius norm of a matrix-matrix product; and the last equality holds because $\| \bP\|_2 = 1$ (Assumption $(A_1)$).
By Assumption $(A_5)$, $\| \E^0 \|^2_F \leq \epsilon $, which completes the proof.
\end{proof}

\begin{lemma}{\textbf (Approximation error for $\bW$)} For $t=1$ in Algorithm~\ref{alg:alternate} and under Assumptions $(A_1) - (A_5)$, the Frobenius norm of the approximation error of the estimated transform with respect to $\Wopt$ is bounded as
$$\| \bW^1 - \Wopt \|_F \leq q \epsilon, $$
\label{lem:w}
where $q$ is a scalar coefficient as in Theorem~\ref{thm:main1}.
\end{lemma}
\begin{proof}
Denote the SVD of $\Zopt {\Z^1}^T$ as $\U_z^1 \bSigma_z^1 {\V_z^1}^T$. From Algorithm~\ref{alg:alternate}, we have
\begin{align*}
\bW^1 = \V^1{\U^1}^T, & \,\,\;\; \bP {\Z^1}^T = \U^1 \bSigma^1 {\V^1}^T.
\end{align*}
Using the SVD of $\Zopt {\Z^1}^T$, we rewrite the above equations as 
\begin{align} 
\bP {\Z^1}^T &\stackrel{(A_1)}{=} \Wopt^T \Zopt {\Z^1}^T = \underbrace{\Wopt^T \U_z^1}_{\U^1} \underbrace{\bSigma_z^1}_{\bSigma^1} \underbrace{{\V_z^1}^T}_{{\V^1}^T} \nonumber \\
\bW^1 & = \V_z^1{\U_z^1}^T \Wopt. \label{eq:w1}
\end{align}

Now the error between $\bW^1$ and $\Wopt$ satisfies
\begin{align}
\| \bW^1 - \Wopt \|^2_F &\stackrel{(Eq. \ref{eq:w1})}{=} \| \V_z^1{\U_z^1}^T \Wopt - \Wopt \|^2_F = \| (\V_z^1{\U_z^1}^T  - \Id) \Wopt \|^2_F = \| \V_z^1{\U_z^1}^T  - \Id\|^2_F, \label{eq:errorw1}
\end{align}
where the matrix $\V_z^1{\U_z^1}^T$ can be further rewritten as follows:
\begin{align}
\V_z^1{\U_z^1}^T &= \V_z^1 \begin{pmatrix}
\bSigma^{1}_z
\end{pmatrix}^{-1} {\U_z^1}^T\U_z^1 \bSigma^{1}_z {\U_z^1}^T \nonumber \\
& = \underbrace{(\Zopt {\Z^1}^T)^{-1}}_{(a)} \underbrace{(\Zopt {\Z^1}^T {\Z^1} \Zopt^T )^{\frac{1}{2}}}_{(b)}.
\label{eq:vu}
\end{align}
The above equality holds for all $\epsilon < 1$, which suffices to ensure $\Zopt {\Z^1}^T$ is invertible. Note that the matrix square root (i.e., the matrix $\mathbf{B}$ in the decomposition $\mathbf{B}^{2}=\mathbf{A}$) in (b) above is the \emph{positive definite square root}.

Using Taylor Series Expansions for the matrix inverse and positive-definite square root along with \eqref{eq:deltaz} and the assumption $\Zopt \Zopt^{T} = \Id$, we have that
\begin{align}
(a) &= (\Zopt {\Z^1}^T)^{-1}\stackrel{(Eq. \ref{eq:deltaz})}{=} (\Zopt(\Zopt + \mathbf{\Delta}^1)^{T})^{-1} = (\Zopt\Zopt^T + \Zopt {\mathbf{\Delta}^1}^T)^{-1} \nonumber \\  
&\stackrel{(A_4)}{=} (\Id + \Zopt {\mathbf{\Delta}^1}^T)^{-1} = \Id - \Zopt {\mathbf{\Delta}^1}^T + O ((\mathbf{\Delta}^1)^2) \nonumber \\
(b) &= (\Zopt {\Z^1}^T {\Z^1} \Zopt^T )^{\frac{1}{2}} \stackrel{(A_4)}{=} \Id + \frac{1}{2}(\Zopt {\mathbf{\Delta}^1}^T + \mathbf{\Delta}^1 \Zopt^T) + O((\mathbf{\Delta}^1)^2) \nonumber \\
\V_z^1{\U_z^1}^T &\stackrel{(Eq. \ref{eq:vu})}{=} (a)(b) \nonumber \\
&= \begin{pmatrix} \Id - \Zopt {\mathbf{\Delta}^1}^T + O ((\mathbf{\Delta}^1)^2) \end{pmatrix} \begin{pmatrix} \Id + \frac{1}{2}(\Zopt {\mathbf{\Delta}^1}^T + \mathbf{\Delta}^1 \Zopt^T) + O((\mathbf{\Delta}^1)^2) \end{pmatrix} \nonumber \\ 
&= \Id + \frac{1}{2}(\mathbf{\Delta}^1 \Zopt^T - \Zopt {\mathbf{\Delta}^1}^T ) + O((\mathbf{\Delta}^1)^2), \label{taylorser}
\end{align}
where $O((\mathbf{\Delta}^1)^2)$ denotes corresponding higher order series terms, and is bounded in norm by $C \begin{Vmatrix}
\mathbf{\Delta}^{1}
\end{Vmatrix}^{2}$ for some constant $C$. 

Substituting these expressions in \eqref{eq:errorw1}, the error between the first transform iterate $\bW^1$ and $\Wopt$ is bounded as
\begin{align}
\|\bW^1 - \Wopt \|_F &\stackrel{( Eq. \ref{eq:errorw1})}{=} \| \V_z^1{\U_z^1}^T  - \Id\|_F \approx \frac{1}{2} \| \mathbf{\Delta}^1 \Zopt^T - \Zopt {\mathbf{\Delta}^1}^T \|_F. \label{upb1}
\end{align}
The approximation error above is bounded in norm
by $C \epsilon^2$, which is negligible for small $\epsilon$. So we only bound the dominant term $0.5 \| \mathbf{\Delta}^1 \Zopt^T - \Zopt {\mathbf{\Delta}^1}^T \|_F$ on the right.
The matrix $\mathbf{\Delta}^1 \Zopt^T - \Zopt {\mathbf{\Delta}^1}^T$ clearly has a zero diagonal (skew-symmetric). Thus, we have the following inequalities:
\begin{align}
\|\bW^1 - \Wopt \|_F &  \approx \frac{1}{2} \| \mathbf{\Delta}^1 \Zopt^T - \Zopt {\mathbf{\Delta}^1}^T \|_F \leq  \sqrt{ \sum_{k=1}^{n} \|  \D_k \Zopt \tilde{\D}_k {\mathbf{\Delta}_{(k,\cdot)}^1}^T \|_2^2 } \nonumber \\
& \leq \sqrt{ \sum_{k=1}^{n} \|  \D_k \Zopt \tilde{\D}_k \|_2^2 \| \mathbf{\Delta}_{(k,\cdot)}^1 \|_2^2} \leq \max_k  \|  \D_k \Zopt \tilde{\D}_k\|_2 \sqrt{ \sum_{k=1}^{n} \| \mathbf{\Delta}_{(k,\cdot)}^1 \|_2^2  } \nonumber \\
& =  \max_k \|  \D_k \Zopt \tilde{\D}_k\|_2 \| \mathbf{\Delta}^1 \|_F \stackrel{Lem. \ref{lem:z}}{\leq} q \| \E^0 \|_F \nonumber \\
& = q  \| \bW^0 - \Wopt \|_F, \label{qfacd1}
\end{align}
where we more simply write
(ignoring higher order terms in \eqref{upb1})
$q:=  max_k \|  \D_k \Zopt \tilde{\D}_k\|_2$.
Since $\| \E^0 \|_F \leq \epsilon$ by Assumption $(A_5)$, we obtain the desired result.
\end{proof}

Thus, we have shown the results for the $t=1$ case. We complete the proof of Theorem~\ref{thm:main1} by observing that for each subsequent iteration $t = \tau+1$, the same steps as above can be repeated along with the induction hypothesis (IH) to show that 
\begin{align*}
\|\Z^{\tau+1} - \Zopt \|_F &= \| \mathbf{\Delta}^{\tau+1} \|_F \leq \| \E^{\tau}\|_F \\
\hskip0.35\textwidth & = \| \bW^{\tau} - \Wopt \|_F  \stackrel{(IH)}{\leq} q^\tau \epsilon\\
\| \bW^{\tau+1} - \Wopt \|_F &\leq q \| \Z^{\tau + 1} - \Zopt \|_F  \leq q(q^{\tau} \epsilon). \hskip0.35\textwidth \square 
\end{align*}

\subsubsection{Proof of Theorem~\ref{thm:main2}}\label{thmmain2proof}
Here, we present the distinctions in the proof of Theorem~\ref{thm:main2}.
When Assumption $(A_4)$ is dropped, Lemma~\ref{lem:z} and its proof remain unaffected. The change to Lemma~\ref{lem:w} and its proof are outlined next.

\begin{lemma}{\textbf (Removing Assumption $(A_4)$)} For $t=1$ in Algorithm~\ref{alg:alternate} and under Assumptions $(A_1) - (A_3)$ and $(A_5)$, the Frobenius norm of the approximation error of the estimated transform with respect to $\Wopt$ is bounded as
$$\| \bW^1 - \Wopt \|_F \leq q \epsilon, $$
where $q$ is a scalar coefficient as in Theorem~\ref{thm:main2}.
\label{lem:z_not_orth}
\end{lemma}
\begin{proof}
The proof of Lemma~\ref{lem:z_not_orth} relies on the general Taylor Series Expansions for the matrix inverse and positive definite square root. In particular, \eqref{taylorser} uses these expansions under the assumption that $\Zopt \Zopt^T = \Id$. To establish a result without this assumption, we first use the general Taylor Series Expansions for matrix inverse and square root then rely on algebraic identities of the Kronecker sum and product to manipulate the error bound of $\| \bW^1 - \Wopt \|_F$. 

To that end, let $\mathbf{G} \triangleq \Zopt \Zopt^T$. First, we look at the series expansion of $(\Zopt {\Z^1}^T)^{-1}$, for which the following equalities hold: 
\begin{align*}
(\Zopt {\Z^1}^T)^{-1} &= (\Zopt \Zopt^T + \Zopt \textcolor{black}{{\boldsymbol{\Delta}^1}^T})^{-1} \\
&=  (\mathbf{G} ( \Id +\mathbf{G}^{-1} \Zopt {\boldsymbol{\Delta}^1}^T) )^{-1} =  ( \Id +\mathbf{G}^{-1} \Zopt {\boldsymbol{\Delta}^1}^T )^{-1} \mathbf{G}^{-1}\\
&=  ( \Id  - \mathbf{G}^{-1} \Zopt {\boldsymbol{\Delta}^1}^T + O((\boldsymbol{\Delta}^1)^2) ) \mathbf{G}^{-1}\\
&=  \mathbf{G}^{-1}  - \mathbf{G}^{-1} \Zopt {\boldsymbol{\Delta}^1}^T \mathbf{G}^{-1} + O((\boldsymbol{\Delta}^1)^2),
\end{align*}
where we factored out\footnote{Matrix $\mathbf{G}$ must be invertible for $\kappa^2(\mathbf{G}) = \kappa^4(\Zopt)$ to be finite and Assumption $(A_3)$ to hold.} $\mathbf{G}^{-1}$ and then computed the series expansion of a matrix inverse. The 
taylor series converges when $ \| \mathbf{G}^{-1} \Zopt \boldsymbol{\Delta}^{1^T}\|_2
\leq \begin{Vmatrix}
\mathbf{G}^{-1}
\end{Vmatrix}_{2} \begin{Vmatrix}
 \Zopt 
\end{Vmatrix}_{2} \begin{Vmatrix}
\boldsymbol{\Delta}^{1} 
\end{Vmatrix}_{F} \leq \kappa^2(\Zopt) \epsilon < 1$
%result holds 
or when $\epsilon < \kappa^{-2}(\Zopt)$.

For the series expansion of the matrix square root in \eqref{eq:vu}, we first observe that
\begin{align*}
(\Zopt {\Z^1}^T \Z^1 {\Zopt}^T)^\frac{1}{2} &= (\Zopt {(\Zopt + \boldsymbol{\Delta}^1)}^T (\Zopt + \boldsymbol{\Delta}^1) {\Zopt}^T)^\frac{1}{2} \\
 &= \left( \Zopt \Zopt^T \Zopt \Zopt^T + (\Zopt \Zopt^T \boldsymbol{\Delta}^1 \Zopt^T + \Zopt {\boldsymbol{\Delta}^1}^T \Zopt \Zopt^T + \Zopt {\boldsymbol{\Delta}^1}^T\boldsymbol{\Delta}^1 \Zopt^T) \right)^\frac{1}{2} \\
 &= \left( \mathbf{G}^2 + (\mathbf{G} \boldsymbol{\Delta}^1 \Zopt^T + \Zopt {\boldsymbol{\Delta}^1}^T \mathbf{G} + \Zopt {\boldsymbol{\Delta}^1}^T\boldsymbol{\Delta}^1 \Zopt^T) \right)^\frac{1}{2}.
\end{align*}
Let $F(\mathbf{G}) \triangleq \left( \mathbf{G}^2 + (\mathbf{G} \boldsymbol{\Delta}^1 \Zopt^T + \Zopt {\boldsymbol{\Delta}^1}^T \mathbf{G} + \Zopt {\boldsymbol{\Delta}^1}^T\boldsymbol{\Delta}^1 \Zopt^T) \right)^\frac{1}{2} = \left( \mathbf{G}^2 + \tilde{\Delta} \right)^\frac{1}{2} $, where $\tilde{\Delta}$ denotes the remainder of terms within the square root. The Taylor Series Expansion for $F(\mathbf{G})$ can be written as 
$F(\G) \approx \G + R^{T}(\nabla F(\mathbf{G}) \mathrm{Vec}(\tilde{\Delta}^{T})) + O(\tilde{\Delta}^2)$, where the operator $\mathrm{Vec}(\cdot)$ reshapes a matrix into a vector by stacking the columns, $R(\cdot)$ undoes or inverts the $\mathrm{Vec}(\cdot)$ operation by reshaping a vector into an $n \times n$ matrix, and the gradient of the square root function is obtained as follows, where $\otimes$ denotes the Kronecker product and $\oplus$ denotes the Kronecker sum:
\begin{equation}
\nabla F(\mathbf{G}) = \frac{\partial \mathrm{Vec}(F^{T}(\mathbf{G}))}{\partial \mathrm{Vec}^{T}(\mathbf{G}^{T})}=\left ( \Id\otimes \G + \G \otimes \Id \right )^{-1}= (\G \oplus \G)^{-1}. \label{sqrtgrad}
\end{equation}

Using the above expressions, \eqref{eq:vu} in this case becomes
\begin{align}
\V_z^1{\U_z^1}^T  & = (\Zopt {\Z^1}^T)^{-1} (\Zopt {\Z^1}^T {\Z^1} \Zopt^T )^{\frac{1}{2}} \nonumber\\
& = \Id - \G^{-1} \Zopt {\boldsymbol{\Delta}^1}^T + \G^{-1} R^{T}((\G \oplus \G)^{-1} \mathrm{Vec}(\tilde{\Delta}^{T})) + O((\boldsymbol{\Delta}^1)^2),\nonumber\\
& = \Id - \G^{-1} \Zopt {\boldsymbol{\Delta}^1}^T + \G^{-1} R^{T}((\G \oplus \G)^{-1} \mathrm{Vec}(\mathbf{G} \boldsymbol{\Delta}^1 \Zopt^T + \Zopt {\boldsymbol{\Delta}^1}^T \mathbf{G})) + O((\boldsymbol{\Delta}^1)^2),
\label{eq:vu2}
\end{align}
with $O((\boldsymbol{\Delta}^1)^2)$ denoting corresponding higher order series terms in each step above.

Now recall from \eqref{upb1} that $\|\bW^1 - \Wopt \|_F = \| \V_z^1{\U_z^1}^T  - \Id\|_F = \| \B \|_F$, where $\B \triangleq \V_z^1{\U_z^1}^T  - \Id$ $=- \G^{-1} \Zopt {\boldsymbol{\Delta}^1}^T + \G^{-1} R^{T}((\G \oplus \G)^{-1} \mathrm{Vec}(\mathbf{G} \boldsymbol{\Delta}^1 \Zopt^T + \Zopt {\boldsymbol{\Delta}^1}^T \mathbf{G})) + O((\boldsymbol{\Delta}^1)^2)$.
To bound the required error $\| \B \|_F$, first, using the property of the $\mathrm{Vec}(\cdot)$ operator that $\mathrm{Vec}(\mathbf{A}\mathbf{X}\mathbf{C})= (\mathbf{C}^{T} \otimes \mathbf{A}) \mathrm{Vec}(\mathbf{X})$, we can easily obtain a simplified expression for $\B$ ignoring the $O((\boldsymbol{\Delta}^1)^2)$ terms (since they are bounded in norm by $C \epsilon^2$, which is negligible for small $\epsilon$ and $C$ is a constant) in \eqref{eq:vu2} as follows:
\begin{align}
\mathrm{Vec}(\B^{T})  = & -(\G^{-1} \otimes \Id) \mathrm{Vec}(\boldsymbol{\Delta}^1\Zopt^{T})
+(\G^{-1} \otimes \Id) (\G \oplus \G)^{-1} (\Id \otimes \G) \mathrm{Vec}(\boldsymbol{\Delta}^1\Zopt^{T}) \nonumber \\
& +(\G^{-1} \otimes \Id) (\G \oplus \G)^{-1} (\G \otimes \Id) \mathrm{Vec}(\Zopt {\boldsymbol{\Delta}^1}^T). \label{vecb}
\end{align}
Denoting the SVD of (positive-definite) $\G$ as $\Q \La \Q^{T}$, it can be shown that the SVD of the Kronecker sum $\G \oplus \G$ is\footnote{The SVD of the Kronecker sum is established by the following equalities that use the definitions of the Kronecker sum and SVD of $\G$ and \eqref{kronprd}: $\G \oplus \G = \Id\otimes \G + \G \otimes \Id = \Q \, \Id \, \Q^{T} \otimes \Q \La \Q^{T} + \Q \La \Q^{T} \otimes \Q \, \Id \, \Q^{T} = (\Q \otimes \Q) (\Id \otimes \La)(\Q^{T} \otimes \Q^{T}) + (\Q \otimes \Q) (\La \otimes \Id)(\Q^{T} \otimes \Q^{T}) =(\Q \otimes \Q) (\La \oplus \La) (\Q^{T} \otimes \Q^{T})$.} $(\Q \otimes \Q) (\La \oplus \La) (\Q \otimes \Q)^{T}$, or that $(\G \oplus \G)^{-1} = (\Q \otimes \Q) (\La \oplus \La)^{-1} (\Q \otimes \Q)^{T}$.
Using these SVDs and the standard result that
\begin{equation}
(\mathbf{H}_{1} \otimes \mathbf{H}_{2})(\mathbf{H}_{3} \otimes \mathbf{H}_{4}) = (\mathbf{H}_{1} \mathbf{H}_{3} \otimes \mathbf{H}_{2} \mathbf{H}_{4}), \label{kronprd}
\end{equation}
the following results readily hold:
\begin{align}
(\G \oplus \G)^{-1} (\Id \otimes \G) & =  (\Q \otimes \Q) (\La \oplus \La)^{-1} (\Q^{T} \otimes \Q^{T}) (\Id \otimes \G)  \nonumber \\
& = (\Q \otimes \Q) (\La \oplus \La)^{-1} (\Q^{T} \otimes \Q^{T} \G)   \nonumber \\
& = (\Q \otimes \Q) (\La \oplus \La)^{-1} (\Q^{T} \otimes \La \Q^{T})   \nonumber \\
& = (\Q \otimes \Q) (\La \oplus \La)^{-1} (\Id \otimes \La) (\Q^{T} \otimes \Q^{T}), 
\label{Qres1}
\end{align}
\vspace{-0.25in}
\begin{align}
(\G \oplus \G)^{-1} (\G \otimes \Id) & =  (\Q \otimes \Q) (\La \oplus \La)^{-1} (\Q^{T}\G \otimes \Q^{T})   \nonumber \\
& = (\Q \otimes \Q) (\La \oplus \La)^{-1} (\La \Q^{T} \otimes \Q^{T}) \nonumber \\
&= (\Q \otimes \Q) (\La \oplus \La)^{-1} (\La \otimes \Id) (\Q^{T} \otimes \Q^{T}).  
\label{Qres2}
\end{align}

Substituting \eqref{Qres1} and \eqref{Qres2} in \eqref{vecb} simplifies \eqref{vecb} as follows:
\begin{multline}
\mathrm{Vec}(\B^{T})  = (\G^{-1} \otimes \Id)(\Q \otimes \Q)\Biggl( 
\begin{bmatrix}
(\La \oplus \La)^{-1}(\Id \otimes \La) - \Id
\end{bmatrix}(\Q^{T} \otimes \Q^{T})\mathrm{Vec}(\boldsymbol{\Delta}^1\Zopt^{T}) \\
 + (\La \oplus \La)^{-1} (\La \otimes \Id) (\Q^{T} \otimes \Q^{T}) \mathrm{Vec}(\Zopt {\boldsymbol{\Delta}^1}^T)  
  \Biggr).
\label{vecb2}
\end{multline}
Moreover, we have that
\begin{align} 
(\La \oplus \La)^{-1}(\Id \otimes \La) - \Id & = (\La \oplus \La)^{-1}\begin{pmatrix}
(\Id \otimes \La) - (\La \oplus \La)
\end{pmatrix} \nonumber \\
& =(\La \oplus \La)^{-1}\begin{pmatrix}
(\Id \otimes \La) - (\Id \otimes \La + \La \otimes \Id)
\end{pmatrix}
= - (\La \oplus \La)^{-1}(\La \otimes \Id).
\end{align}
Thus, equation \eqref{vecb2} further simplifies to
\begin{equation}
\mathrm{Vec}(\B^{T})  = \mathbf{H} \,
\mathrm{Vec}\left ( \Zopt {\boldsymbol{\Delta}^1}^T - \boldsymbol{\Delta}^1\Zopt^{T} \right ),
\label{vecb3}
\end{equation}
where the matrix $\mathbf{H}$ is defined as
\begin{equation}
\mathbf{H} \triangleq (\G^{-1} \otimes \Id)(\Q \otimes \Q)(\La \oplus \La)^{-1}(\La \otimes \Id)(\Q^{T} \otimes \Q^{T}).
\end{equation}

Finally, we use \eqref{vecb3} to obtain
\begin{equation}
\|\bW^1 - \Wopt \|_F = \| \V_z^1{\U_z^1}^T  - \Id\|_F \approx \| \mathbf{H} \,
\mathrm{Vec}\left ( \Zopt {\boldsymbol{\Delta}^1}^T - \boldsymbol{\Delta}^1\Zopt^{T} \right ) \|_2 \leq \| \mathbf{H} \|_2 \| \boldsymbol{\Delta}^1\Zopt^{T} - \Zopt {\boldsymbol{\Delta}^1}^T  \|_{F}. \label{vecw4}
\end{equation}
Here, the submultiplicativity of the spectral norm and the fact that $\left \| \mathbf{H}_{1} \otimes \mathbf{H}_{2} \right \|_{2} = \left \| \mathbf{H}_{1} \right \|_{2} \left \| \mathbf{H}_{2} \right \|_{2}$ ensures that
\begin{align}
\| \mathbf{H} \|_2 & \leq \begin{Vmatrix}\G^{-1} \otimes \Id\end{Vmatrix}_{2}
\begin{Vmatrix}\Q \otimes \Q\end{Vmatrix}_{2}
\begin{Vmatrix}(\La \oplus \La)^{-1}\end{Vmatrix}_{2}
\begin{Vmatrix}\La \otimes \Id \end{Vmatrix}_{2}
\begin{Vmatrix}\Q^{T} \otimes \Q^{T}\end{Vmatrix}_{2} \nonumber \\
& = \begin{Vmatrix}\G^{-1}\end{Vmatrix}_{2}\begin{Vmatrix}\Q \end{Vmatrix}_{2}^{2}
\begin{Vmatrix}(\La \oplus \La)^{-1}\end{Vmatrix}_{2}\begin{Vmatrix}\La \end{Vmatrix}_{2}
\begin{Vmatrix}\Q^{T}\end{Vmatrix}_{2}^{2} = \frac{\kappa^{4}(\Zopt)}{2}, \label{hineq1}
\end{align}
where the last equality follows from the facts that $\begin{Vmatrix}\Q \end{Vmatrix}_{2}=1$ (for unitary matrix); $\begin{Vmatrix}\La \end{Vmatrix}_{2} = \begin{Vmatrix}\G \end{Vmatrix}_{2} = \begin{Vmatrix}\Zopt \end{Vmatrix}_{2}^{2} = \begin{Vmatrix}\mathbf{P} \end{Vmatrix}_{2}^{2} = 1$ (by Assumption ($A_1$)); $\begin{Vmatrix}(\La \oplus \La)^{-1}\end{Vmatrix}_{2} = 0.5 \begin{Vmatrix}\G^{-1}\end{Vmatrix}_{2} = 0.5\sigma_{n}^{-1}(\G)= 0.5\sigma_{n}^{-2}(\Zopt)$, where $\sigma_{n} (\cdot)$ denotes the smallest matrix singular value; and the fact that $\kappa(\Zopt) = \sigma_{1}(\Zopt)/\sigma_{n}(\Zopt)= \sigma_{n}^{-1}(\Zopt)
$ (using Assumption ($A_1$)).
Substituting \eqref{hineq1} in \eqref{vecw4} and using a similar set of inequalities as in \eqref{qfacd1} to bound the $\| \boldsymbol{\Delta}^1\Zopt^{T} - \Zopt {\boldsymbol{\Delta}^1}^T  \|_{F}$ term in \eqref{vecw4} provides the following bound:
\begin{equation}
\|\bW^1 - \Wopt \|_F \approx \| \mathbf{H} \,
\mathrm{Vec}\left ( \Zopt {\boldsymbol{\Delta}^1}^T - \boldsymbol{\Delta}^1\Zopt^{T} \right ) \|_2 \leq  q  \| \bW^0 - \Wopt \|_F,
 \label{qfacd2}
\end{equation}
where we more simply write $q \, := \, \kappa^{4}(\Zopt) \, \textcolor{black}{\max_k} \|  \D_k \Zopt \tilde{\D}_k\|_2$.
Since by Assumption $(A_5)$, $\| \E^0 \|_F \leq \epsilon$, we obtain the desired result.
\end{proof}

\subsubsection{Proof of Corollary~\ref{coro1}}
We have $q = \max_k \|  \D_k \Zopt \tilde{\D}_k\|_2$ (focusing on the dominant component) with~$\Zopt \Zopt^{T} = \Id$ by assumptions ($A_3$) and ($A_4$), respectively. For brevity in notation, let $\M_k = \D_k \Zopt \tilde{\D}_k$. Here the matrix $\D_k$ zeros out the $k$th row of $\Zopt$ and $\tilde{\D}_k$ zeros out the columns corresponding to the complement of the support of the $k$th row of $\Zopt$. The matrix $\M_k \M_k^T$ is then a diagonal matrix where the $(k,k)^{\mathrm{th}}$ entry is $0$ and the $(i,i)^{\mathrm{th}}$ entry for $i \neq k$ is $\| \row{\Zopt} |_{S(
\Zopt_{(i,.)}) \cap S(
\Zopt_{(k,.)})} \|^2_2$,
where $\row{\Zopt} |_{S(
\Zopt_{(i,.)}) \cap S(
\Zopt_{(k,.)})}$ coincides with $\row{\Zopt}$ on $S(
\Zopt_{(i,.)}) \cap S(
\Zopt_{(k,.)})$ and is zero outside this support.
Clearly, the $k$th row and column of $\M_k \M_k^T$ are zero and its other off-diagonal entries are
$\langle \Zopt_{(i,.)}|_{S(
\Zopt_{(i,.)}) \cap S(
\Zopt_{(k,.)})}, \rowj{\Zopt}|_{S(
\Zopt_{(j,.)}) \cap S(
\Zopt_{(k,.)})}\rangle$ $= 0$ because each column of $\Zopt$ has at most $s=2$ non-zeros and $S(
\Zopt_{(i,.)}) \cap S(
\Zopt_{(j,.)}) \cap S(
\Zopt_{(k,.)}) = \emptyset$ for $i \neq j \neq k$. So, we readily have that 
\begin{align*}
q^{2} &= \max_k \|  \M_k \M_k^T \|_2 \\
& = \max_{1\leq k \leq n} \, \max_{i \neq k} \| \row{\Zopt} |_{S(
\Zopt_{(i,.)}) \cap S(
\Zopt_{(k,.)})} \|^2_2 \\
& < 1,
\end{align*}
where the last inequality bound follows from the fact that $\| \row{\Zopt} |_{S(
\Zopt_{(i,.)}) \cap S(
\Zopt_{(k,.)})} \|^2_2 < 1$ for all $i \neq k$, which holds because each row of $\Zopt$ has unit $\ell_{2}$ norm (assumption $(A_4)$) and no two rows have the exact same support. \hfill $\square$

\subsubsection{Proof of Proposition~\ref{conj}} \label{secproof2}
Under the conditions stated in Proposition~\ref{conj}, the (dominant) $q$ factor is expected to be less than $1$ given sufficient training signals, i.e., large $N$.
For the proof, we study the asymptotic behavior of the matrices $\mathbf{H}\triangleq \Zopt \Zopt^{T}$ and 
$\mathbf{G} \triangleq \M_{k}  \M_{k}^{T}$, where $\M_{k} \triangleq \D_k \Zopt \tilde{\D}_k$, which appear in $q_{N} = \begin{pmatrix}
\kappa^{4}\begin{pmatrix}
\Zopt
\end{pmatrix}/\| \bP \|_{2}
\end{pmatrix}\max_{1\leq k \leq n} \| \D_k \Zopt \tilde{\D}_k \|_2$ as defined in Remark~\ref{rm11}. First, we show that $\begin{pmatrix}
\kappa^{4}\begin{pmatrix}
\Zopt
\end{pmatrix}/\| \bP \|_{2}
\end{pmatrix} \rightarrow 1$ almost surely as $N \rightarrow \infty$ using $\mathbf{H}$. Then, we will show that $\| \D_k \Zopt \tilde{\D}_k \|_2$ $\rightarrow \sqrt{s-1/n-1}$ almost surely as $N \rightarrow \infty$ using $\mathbf{G}$.

Let $\tilde{\Z}^{*} = \sqrt{N}\Zopt$. Then the nonzero entries of $\tilde{\Z}^{*}$ have zero mean and variance of $n/s$.
Let $\mathds{1}_{\left \{ j \in S(
\Zopt_{(.,l)}) \right \}}$ denote the indicator function that takes the value $1$ when $j \in S(\Zopt_{(.,l)})$ and is zero otherwise.
Since $\Zopt \Zopt^{T} = N^{-1} \tilde{\Z}^{*} \tilde{\Z}^{*^{T}}$, using the law of large numbers, the diagonal entries of $\mathbf{H}$ converge almost surely as follows:
\begin{equation}
\lim_{N \to \infty} \mathbf{H}_{(j,j)} = \lim_{N \to \infty}\frac{1}{N}\sum_{l=1}^{N}\tilde{\Z}^{*^{2}}_{(j,l)} \mathds{1}_{\left \{ j \in S(
\Zopt_{(.,l)}) \right \}} = E\left [ \tilde{\Z}^{*^{2}}_{(j,l)} \mathds{1}_{\left \{ j \in S(
\Zopt_{(.,l)}) \right \}} \right ]=1,
\label{hlim1}
\end{equation}
where $b \triangleq \tilde{\Z}^{*^{2}}_{(j,l)} \mathds{1}_{\left \{ j \in S(
\Zopt_{(.,l)}) \right \}}$ is i.i.d. over the columns $l$. The random variable $b$ is nonzero (the nonzero part has mean $n/s$) with probability (w.p.)\footnote{The probability that $j \in S\begin{pmatrix}
\Zopt_{(.,l)}
\end{pmatrix}$ is $\frac{\binom{n-1}{s-1}}{\binom{n}{s}}=\frac{s}{n}$.}
$s/n$ and is zero w.p. $1 -(s/n)$, implying $E\left [ b \right ] = 1$.
Similarly, the off-diagonal entries $\mathbf{H}_{(i,j)}$ for $i \neq j$ converge 
as follows:
\begin{equation}
\lim_{N \to \infty} \mathbf{H}_{(i,j)} = \lim_{N \to \infty}\frac{1}{N}\sum_{l=1}^{N}\tilde{\Z}^{*}_{(i,l)}\tilde{\Z}^{*}_{(j,l)} \mathds{1}_{\left \{ i, j \in S(
\Zopt_{(.,l)}) \right \}} = E\left [ \tilde{\Z}^{*}_{(i,l)}\tilde{\Z}^{*}_{(j,l)} \mathds{1}_{\left \{ i, j \in S(
\Zopt_{(.,l)}) \right \}} \right ]=0,
\label{hlim2}
\end{equation}
where $h \triangleq \tilde{\Z}^{*}_{(i,l)}\tilde{\Z}^{*}_{(j,l)} \mathds{1}_{\left \{ i, j \in S(
\Zopt_{(.,l)}) \right \}}$ is nonzero w.p.\footnote{This is the probability that the two indexes $i$ and $j$ both appear in the support of the $l$th column of $\Zopt$. Thus, $r= \frac{\binom{n-2}{s-2}}{\binom{n}{s}} = \frac{s(s-1)}{n(n-1)}$.} $r = s(s-1)/n(n-1)$ and zero w.p. $1 - r$, implying $E\left [ h \right ] = \begin{pmatrix}
s(s-1)/n(n-1)
\end{pmatrix} E[a] = 0$,
where $a$ is the product of two i.i.d. zero mean random variables.
Therefore, from \eqref{hlim1} and \eqref{hlim2}, it follows that $\mathbf{H} = \Zopt \Zopt^{T}$ converges to $\Id$ almost surely.
Thus, as $N \to \infty$, $\kappa^{4}\begin{pmatrix}
\Zopt
\end{pmatrix}/\| \bP \|_{2} =
\kappa^{2}\begin{pmatrix}
\mathbf{H}
\end{pmatrix}/\sqrt{\| \mathbf{H} \|_{2}}$ in the definition of $q_{N}$, converges to $1$ almost surely.

Now consider $\mathbf{G}$ and note that the $k$th row and column of the matrix $\mathbf{G}$ are zero. As $N \to \infty$, the diagonal entries of $\mathbf{G}$ have the following limit almost surely:
\begin{align}
\lim_{N \to \infty} \mathbf{G}_{(j,j)} &= \lim_{N \to \infty}\frac{1}{N}\sum_{l=1}^{N}\tilde{\Z}^{*^{2}}_{(j,l)} \mathds{1}_{\left \{ l \in S(
\Zopt_{(j,.)}) \cap S(
\Zopt_{(k,.)}) \right \}} \nonumber \\
&= E\left [ \tilde{\Z}^{*^{2}}_{(j,l)} \mathds{1}_{\left \{ l \in S(
\Zopt_{(j,.)}) \cap S(
\Zopt_{(k,.)}) \right \}} \right ] \nonumber \\
&= \frac{s(s-1)}{n(n-1)} \times \frac{n}{s} = \frac{s-1}{n-1},
\label{glim1}
\end{align}
which holds for all $j \neq k$. The expectation follows from the fact that $\tilde{\Z}^{*^{2}}_{(j,l)} \mathds{1}_{\left \{ l \in S(
\Zopt_{(j,.)}) \cap S(
\Zopt_{(k,.)}) \right \}}$ is i.i.d. over the columns\footnote{Note that $\mathds{1}_{\left \{ l \in S(
\Zopt_{(j,.)}) \cap S(
\Zopt_{(k,.)}) \right \}} = \mathds{1}_{\left \{ j, k \in S(
\Zopt_{(.,l)}) \right \}}$.} $l$, is nonzero
(mean $n/s$ for nonzero part)
w.p. $s(s-1)/n(n-1)$, and is zero otherwise.
The following limit holds almost surely for the off-diagonal entries of $\mathbf{G}$:
\begin{align}
\lim_{N \to \infty} \mathbf{G}_{(i,j)} &= \lim_{N \to \infty}\frac{1}{N}\sum_{l=1}^{N}\tilde{\Z}^{*}_{(i,l)} \tilde{\Z}^{*}_{(j,l)} \mathds{1}_{\left \{ l \in S(
\Zopt_{(i,.)}) \cap S(
\Zopt_{(j,.)}) \cap S(
\Zopt_{(k,.)}) \right \}}\nonumber \\
&= E\left [ \tilde{\Z}^{*}_{(i,l)} \tilde{\Z}^{*}_{(j,l)} \mathds{1}_{\left \{ l \in S(
\Zopt_{(i,.)}) \cap S(
\Zopt_{(j,.)}) \cap S(
\Zopt_{(k,.)}) \right \}} \right ]= 0,
\label{glim2}
\end{align}
which follows because the indexes $i$, $j$, and $k$ all lie in the support of the $l$th column (to get non-zero indicator function) w.p. $\frac{s(s-1)(s-2)}{n(n-1)(n-2)}$, and the expectation of the product of zero mean i.i.d. random variables is zero.
It is obvious from \eqref{glim1} and \eqref{glim2} that
\begin{equation}
\lim_{N\to \infty} \M_{k} \M_{k}^{T} = \frac{s-1}{n-1} \D_k \;\; \textrm{a.s.}
\label{glim}
\end{equation}
Thus, as $N \to \infty$, $\left \| \D_k \Zopt \tilde{\D}_k \right \|_{2}$ $= \left \| \M_k \right \|_{2}$ $ \to \sqrt{s-1/n-1}$ almost surely, and the same is true for $\max_{1\leq k \leq n} \| \M_k \|_2$.
Combining all the above results, the required result \eqref{qlim1} is readily established. \hfill $\square$

Note that under the assumed probabilistic model of $\Zopt$, the matrix $\M_{k}\M_{k}^{T}$ in the proof of Proposition~\ref{conj} above approaches a diagonal matrix as $N \to \infty$, whereas in the proof of Corollary~\ref{coro1} for the $s=2$ case, it is deterministically a diagonal matrix for each $N$.

\section{Experiments} \label{sec3}
In this section, we provide numerical results supporting our findings. We also discuss the empirical behavior of the algorithm with respect to different initializations.

\subsection{Empirical Performance of Algorithm}
In the first two experiments, we generated the training set $\bP$ using randomly generated $\Wopt$ and $\Zopt$, and set $n = 50$, $N = 10000$, and $s=\{ 5, \, 10 \}$. The transform $\Wopt$ is generated in each case by applying Matlab's \texttt{orth()} function on a standard Gaussian matrix. For generating $\Zopt$, the support of each column is chosen uniformly at random and the nonzero entries are drawn i.i.d. from a Gaussian distribution with mean zero and variance $n/sN$. Section~\ref{sec2a} (Theorems~\ref{thm:main1} and \ref{thm:main2}) established model recovery guarantees for Algorithm~\ref{alg:alternate}. Figure~\ref{fig:errorW} shows the empirical evolution of the Frobenius norm of the approximation error of the transform iterates with respect to $\Wopt$, for an $\epsilon$ initialization ($\epsilon=0.49 \min_j \beta \left( \col{\Z}^*/\|\col{\Z}^*\|_{2} \right)$ -- see \eqref{eq:zi}). The plots illustrate the observed linear convergence of the iterates to the underlying true
operator $\Wopt$.

\begin{figure}[!h]
\centering
\includegraphics[width=.45\textwidth]{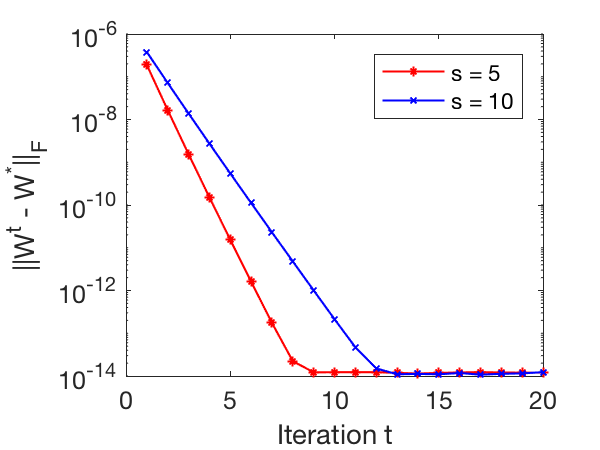}
\caption{The performance of Algorithm \ref{alg:alternate} for recovering $\Wopt$ for $s=5$ and $s=10$.}
\label{fig:errorW}
\end{figure}

Fig.~\ref{fig:init5} and Fig.~\ref{fig:init10} show the behavior of Algorithm~\ref{alg:alternate} with different initializations. We consider six different initializations and plot the evolution of the objective function over iterations. The first initialization, labeled `eps', denotes an initialization as in Fig.~\ref{fig:errorW} with $\epsilon=0.49 \min_j \beta \left( \col{\Z}^*/ \|\col{\Z}^*\|_2\right)$. The other initializations are as follows: entries of $\bW^0$ drawn i.i.d. from a standard Gaussian distribution (labeled `rand'); an $n \times n$ identity matrix $\bW^0$ labeled `id'; a discrete cosine transform (DCT) initialization labeled `dct'; entries of $\bW^0$ drawn i.i.d. from a uniform distribution ranging from 0 to 1 (labeled `unif'); and $\bW^0 = \textbf{0}^{n \times n}$ labeled `zero'. Note that the minimum objective value in \eqref{eq:opt_program} is $0$. For non-epsilon initializations, we see that the behavior of Algorithm~\ref{alg:alternate} is split into two phases. In the first phase, the iterates slowly decrease the objective. When the iterates are close enough to a solution, the second phase occurs and during this phase, Algorithm~\ref{alg:alternate} enjoys rapid convergence (towards 0). For different initializations, the algorithm converged to a scaled (by a diagonal $\pm 1$ matrix), row permuted version of the predetermined $\Wopt$. Fig.~\ref{fig:init5} and Fig.~\ref{fig:init10} also show the proportion of recovered (entry-wise) support of $\Zopt$ (up to row-permutation and sign changes). The grey region highlights the range of iterations in which the true support of $\Zopt$ is estimated well by the different initializations (i.e., where the proportion of recovered support reaches near 1 or 100\%). These empirical results show that the aforementioned second phase of the convergence behavior occurs in the iterations proceeding the point when the algorithm acquires the true support of $\Zopt$. Furthermore, note that the objective's convergence rate in the second phase is similar to that of the `eps' case, where $\epsilon$ is selected to ensure that the support of $\Zopt$ is recovered in one iteration. These results concur with the analysis and discussion in Section~\ref{sec2a}. The behavior of Algorithm~\ref{alg:alternate} is similar for $s=5$ and $s=10$, with the latter case taking more iterations to enter the second phase of convergence. This makes sense since there are more coefficients to learn for larger $s$. This experiment shows that Algorithm~\ref{alg:alternate} is robust to initialization.

\begin{figure}[!h]
\centering
\includegraphics[width=.9\textwidth]{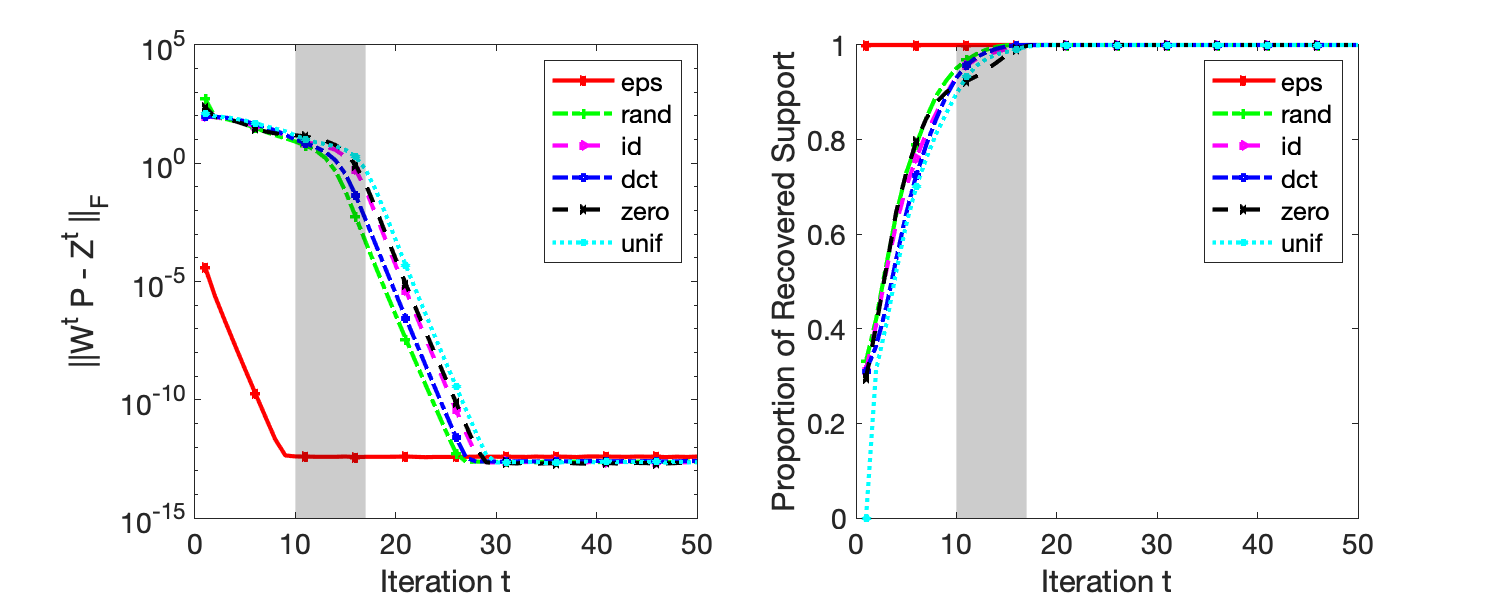}
\caption{The performance of Algorithm~\ref{alg:alternate} over iterations with various initializations for $s=5$: objective function (left) and proportion of recovered support of $\Zopt$ (right).}
\label{fig:init5}
\end{figure}

\begin{figure}[!h]
\centering
\includegraphics[width=.9\textwidth]{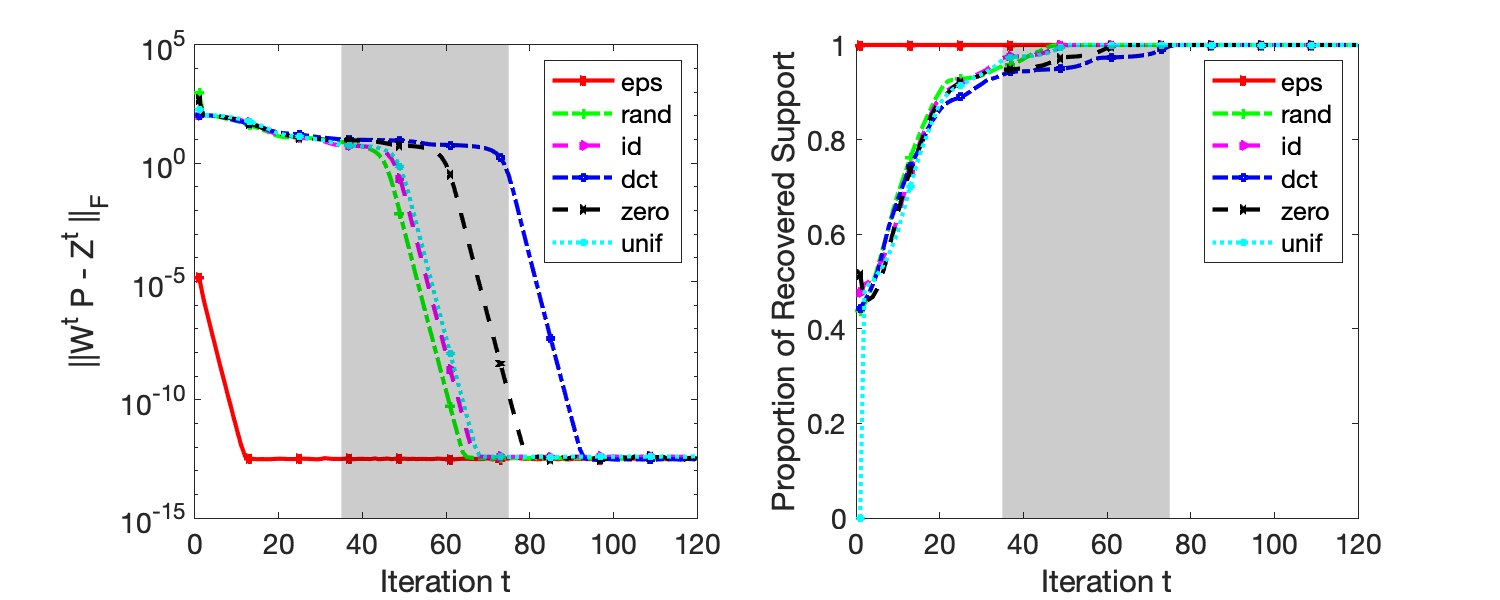}
\caption{The performance of Algorithm~\ref{alg:alternate} with various initializations for $s=10$: objective function (left) and proportion of recovered support of $\Zopt$ (right).}
\label{fig:init10}
\end{figure}

\subsection{The $q$ factor in Proposition~\ref{conj}}
\begin{figure}[h!]
\centering
\includegraphics[width=\textwidth]{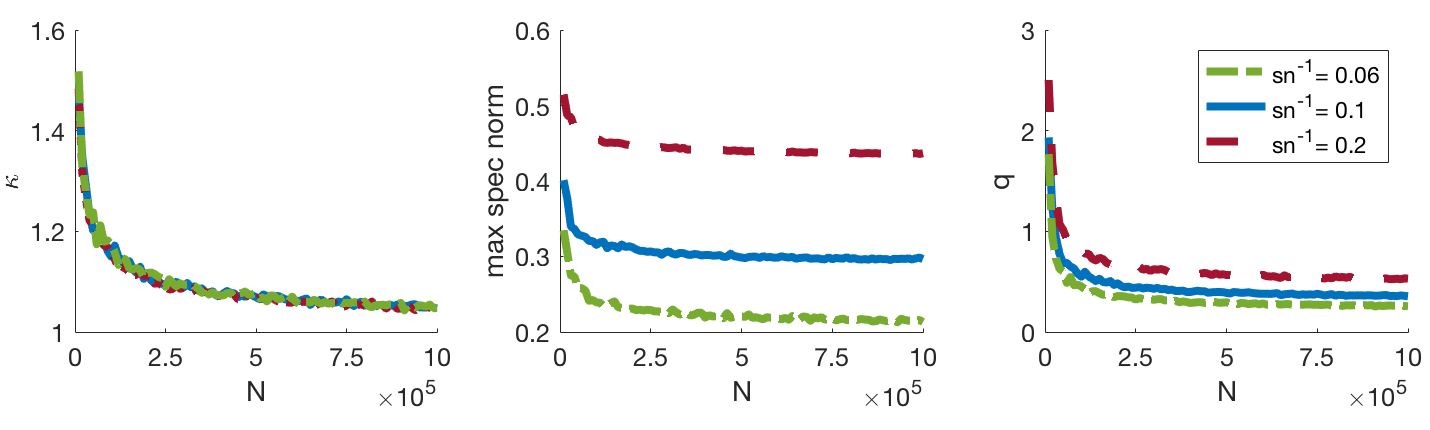}
\includegraphics[width=\textwidth]{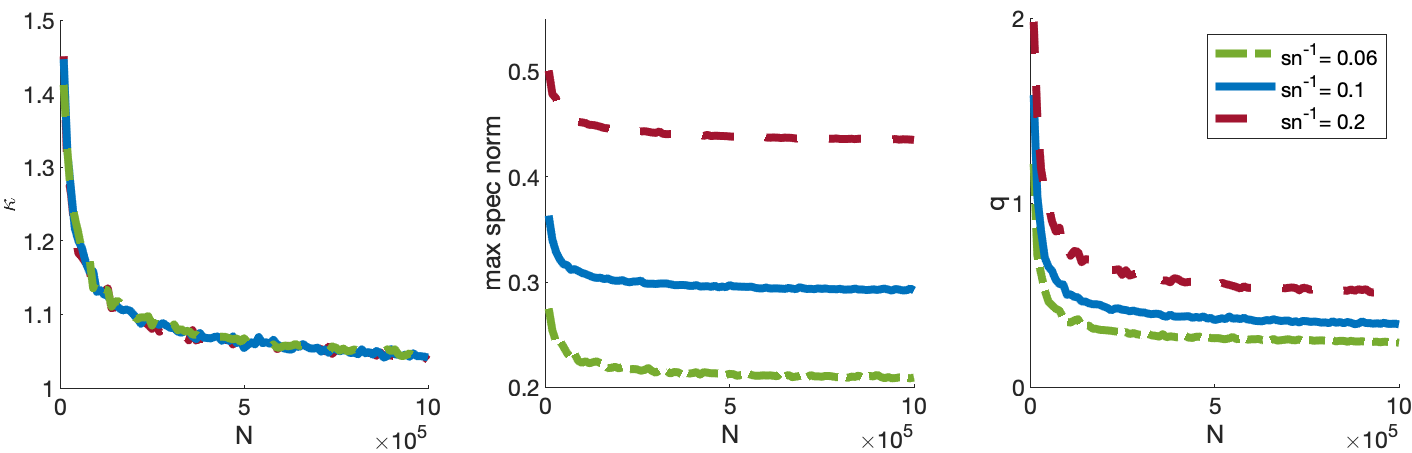}
\caption{On the x-axis we plot the number of training data points $N$ and on the y-axis, (Left) the condition number $\kappa(\Zopt)$, (Center) the maximum spectral norm over choice $k$ for $\D_k \Zopt \tilde{\D}_k$, and (Right) the contraction factor $q = (\kappa^4(\Zopt)/\| \Zopt \|_{2}) \max_k \|\D_k \Zopt \tilde{\D}_k \|_2$. The top row of plots corresponds to the case when the non-zeros are i.i.d. Gaussian and the bottom plots correspond to the non-zeros being i.i.d. scaled random signs. In both cases, $n=50$ and we vary $s/n = \{0.06, 0.1, 0.2\}$.
}
\label{fig:bounds}
\end{figure}
In our last experiment, we illustrate Proposition~\ref{conj} empirically.
For each trial, we fix the signal dimension to be $n = 50$.  
In addition to varying $N$, we vary $s /n = \{0.06, 0.1, 0.2\}$. In the first experiment, the $s$ nonzero entries for each column of $\Zopt$ are selected uniformly at random where values are drawn i.i.d. from a Gaussian distribution with mean $0$ and variance $n / sN$.
We also simulate the case when the nonzeros are i.i.d. scaled random signs with mean $0$ and variance $n / sN$, with ``+" and ``-" being equally probable.
We then compute the following functions of $\Zopt$: the condition number $\kappa(\Zopt)$, the maximum spectral norm over choice $k$ for $\D_k \Zopt \tilde{\D}_k$, and the contraction factor $q = (\kappa^4(\Zopt)/\| \Zopt \|_{2}) \max_k \|\D_k \Zopt \tilde{\D}_k \|_2$ that is a function of these quantities. The top and bottom rows of Fig.~\ref{fig:bounds} plot these quantities for the Gaussian and scaled sign coefficients, respectively. The plots clearly show that $q<1$ for large $N$ for each distribution and $s/n$ setting. The maximum spectral norm plots quickly converged close to their expected values of $\sqrt{s-1/n-1}$. Moreover, $\kappa(\Zopt)$ approaches close to $1$ as $N$ increases as expected, indicating that the probabilistic sparsity model approaches the scenario in Theorem~\ref{thm:main1}. We have observed similar empirical behavior for the $q$ factor, when the non-zero entries are drawn from other distributions. 

\section{Conclusion} \label{sec4}

In this work, we presented a study of the model recovery properties of the alternating minimization algorithm for structured, unitary sparsifying transform learning. The algorithm converges rapidly to the generative model(s) from local neighborhoods under mild assumptions and the assumptions are shown to hold for various probabilistic models. In addition to showing that the algorithm convergences linearly, we also characterize the asymptotic behavior of the convergence rate and radius with respect to the number of data points or training signals $N$. In practice, the sparsifying operator learning method is robust to initialization. Our numerical results and initial analysis showed that the algorithm performs well under various initializations, with similar eventual rates of convergence. We have observed empirically that the algorithm converges to the specific $\Wopt$ even with quite large perturbations for the initial $\bW^0$ from $\Wopt$ (i.e., large $\epsilon$ values in Assumption $(A_5)$). We plan to further analyze the effects of initialization and the behavior of transform learning in inverse problems in future work.

\section*{Funding}
This work 
was partly done when S.R. was at the University of Michigan, Ann Arbor, and
was supported by the 
Office of Naval Research [grant number N00014-15-1-2141 to S.R.]; 
Defense Advanced Research Projects Agency Young Faculty Award [grant number D14AP00086 to S.R.];
US Army Research Office Multidisciplinary University Research Initiative [grant numbers W911NF-11-1-0391, 2015-05174-05 to S.R.]; 
National Institutes of Health [grant numbers R01 EB023618, 
U01 EB018753
to S.R.]; 
National Science Foundation CCF [grant number 1320953 to S.R.];
and the University of Michigan-Shanghai Jiao Tong University seed grant to S.R.
This material was also supported by the National Science Foundation [grant number DMS-1440140 to A.M. and D.N.] while the authors were in residence at the Mathematical Science Research Institute in Berkeley, California, during the Fall 2017 semester, 
National Science Foundation CAREER [grant number 1348721 to A.M. and D.N.]; 
and National Science Foundation BIGDATA  [grant number 1740325 to A.M. and D.N.].

\begin{appendices}
\section{Proofs of Proposition~\ref{convradius} and Remark~\ref{rm1b}} \label{app1}
Here, we present the proof of Proposition~\ref{convradius} and briefly comment on Remark~\ref{rm1b}.
The form of $\epsilon_1$ was discussed in the proof of Lemma~\ref{lem:z} and ensures recovery of the support of $\Zopt$.
We derive the form of $\epsilon_2$ (i.e., sufficient $\epsilon$ for the operator update step) based on the proof of Lemma~\ref{lem:w}.
In particular, we bound $\epsilon$ to ensure convergence of the Taylor Series in
\eqref{taylorser} and to bound the higher order terms in the product $\V_z^1{\U_z^1}^T$.

The matrix inverse series $(\Id + \Zopt {\mathbf{\Delta}^1}^T)^{-1} = \Id - \Zopt {\mathbf{\Delta}^1}^T + O ((\mathbf{\Delta}^1)^2)$ converges when $\left \| \Zopt {\mathbf{\Delta}^1}^T \right \|_{2} <1$. We have $\left \| \Zopt {\mathbf{\Delta}^1}^T \right \|_{2} \leq $ $\begin{Vmatrix}
\Zopt
\end{Vmatrix}_{2} \begin{Vmatrix}
 {\mathbf{\Delta}^1}
\end{Vmatrix}_{2} \leq \begin{Vmatrix}
P
\end{Vmatrix}_{2} \begin{Vmatrix}
 {\mathbf{\Delta}^1}
\end{Vmatrix}_{F} \leq \epsilon$, where the last inequality follows from Assumption $(A_1)$ and Lemma~\ref{lem:z}. Thus, $\epsilon<1$ suffices.
Similarly, the series for $(\Zopt {\Z^1}^T {\Z^1} \Zopt^T )^{\frac{1}{2}} =$ $ (\Id + \Zopt{\mathbf{\Delta}^1}^T + {\mathbf{\Delta}^1}\Zopt^{T} +  \Zopt{\mathbf{\Delta}^1}^T {\mathbf{\Delta}^1} \Zopt^{T}  )^{\frac{1}{2}}$
converges when the perturbation $\mathbf{H} \triangleq \Zopt{\mathbf{\Delta}^1}^T + {\mathbf{\Delta}^1}\Zopt^{T} + \Zopt{\mathbf{\Delta}^1}^T {\mathbf{\Delta}^1} \Zopt^{T}$ satisfies $\left \| \mathbf{H} \right \|_{2}<1$.
Since $\left \| \mathbf{H} \right \|_{2} \leq \left \| \Zopt{\mathbf{\Delta}^1}^T \right \|_{2} + \left \| {\mathbf{\Delta}^1}\Zopt^{T}\right \|_{2} +$ $ \left \| \Zopt{\mathbf{\Delta}^1}^T {\mathbf{\Delta}^1} \Zopt^{T}\right \|_{2} \leq$ $2\left \| \mathbf{\Delta}^1 \right \|_{2} + \left \| \mathbf{\Delta}^1 \right \|_{2}^{2} \leq 2\left \| \mathbf{\Delta}^1 \right \|_{F} + \left \| \mathbf{\Delta}^1 \right \|_{F}^{2} \leq 2 \epsilon + \epsilon^2$, we have that $\epsilon^2 + 2\epsilon <1$ or $\epsilon < \sqrt{2} -1$ suffices, which also works for the matrix inverse series.

Using the 
notation above,
the product in \eqref{taylorser} simplifies as
\begin{align}
    \V_z^1{\U_z^1}^T &= \begin{pmatrix} \Id - \Zopt {\mathbf{\Delta}^1}^T + \mathbf{H}_1\end{pmatrix} \begin{pmatrix} \Id + 0.5\mathbf{H} + \mathbf{H}_2\end{pmatrix} \\
    &= \Id + \frac{1}{2}(\mathbf{\Delta}^1 \Zopt^T - \Zopt {\mathbf{\Delta}^1}^T ) + O((\mathbf{\Delta}^1)^2), \tag{\ref{taylorser}}
\end{align}
where $\mathbf{H}_1$ and $\mathbf{H}_2$ are the remaining higher order terms in the respective series. The $O((\mathbf{\Delta}^1)^2)$ terms in \eqref{taylorser} are given as $0.5 \Zopt {\mathbf{\Delta}^1}^T {\mathbf{\Delta}^1} \Zopt^{T} + \mathbf{H}_{2} - \Zopt {\mathbf{\Delta}^1}^T (0.5 \mathbf{H} + \mathbf{H}_2) + \mathbf{H}_1 (\Id + \mathbf{H} )^{\frac{1}{2}}$.
We bound the Frobenius norm of these summands to characterize $C$ in $
\left \| O((\mathbf{\Delta}^1)^2) \right \|_{F} \leq C \epsilon^2$.

First, we have the following bound:
\begin{equation}
    \left \| 0.5 \Zopt {\mathbf{\Delta}^1}^T {\mathbf{\Delta}^1} \Zopt^{T} \right \|_{F} \leq 0.5 \left \| \Zopt \right \|_{2}^{2} \left \| \mathbf{\Delta}^1 \right \|_{F}^2 \leq 0.5 \epsilon^2. \label{eq:zp1}
\end{equation}
We also have the next bound, which follows from $\left \| \Id + \mathbf{H} \right \|_{2} \leq 1 + \left \| \mathbf{H} \right \|_{2} < 2$ (since $\left \| \mathbf{H} \right \|_{2} < 1$) and $\left \| \mathbf{H}_1 \right \|_{F} \leq $ $ \sum_{j\geq 2} \left \| \mathbf{\Delta}^1 \right \|_{F}^{j}\leq $ $  \sum_{j\geq 2} \epsilon^j = \epsilon^2/(1-\epsilon)$:
\begin{equation}
    \left \| \mathbf{H}_1 (\Id + \mathbf{H} )^{\frac{1}{2}} \right \|_{F} \leq \left \| \mathbf{H}_1 \right \|_{F} \left \| \Id + \mathbf{H} \right \|_{2}^{\frac{1}{2}}\leq \frac{\sqrt{2} \epsilon^2}{1-\epsilon}. \label{eq:zp2}
\end{equation}
Third, we have for the matrix square root Taylor series that $\left \| \mathbf{H}_{2} \right \|_{F} \leq (\left \| \mathbf{H} \right \|_{F}^2 /8) +$ $ (\left \| \mathbf{H} \right \|_{F}^3 /16) + $ $ (5\left \| \mathbf{H} \right \|_{F}^4 /128)+..$, where the right hand side is the magnitude of the remainder of the series for $\sqrt{1-\left \| \mathbf{H} \right \|_{F}}$ after the first order term. Thus, we have the following standard bound for the remainder for some $-\left \| \mathbf{H} \right \|_{F}\leq \alpha \leq 0$:
\begin{equation}
    \left \| \mathbf{H}_{2} \right \|_{F} \leq \frac{\left \| \mathbf{H}\right \|_{F}^{2}}{8(1+ \alpha)^{3/2}} \leq \frac{\left \| \mathbf{H} \right \|_{F}^{2}}{8(1 - \left \| \mathbf{H} \right \|_{F})^{3/2}} \leq \frac{9 \epsilon^2}{8(1 - \epsilon^2 - 2 \epsilon)^{3/2}}.
    \label{eq:zp3}
\end{equation}
Here, the last inequality used $\left \| \mathbf{H} \right \|_{F} \leq \epsilon^2 + 2 \epsilon$ and $\left \| \mathbf{H} \right \|_{F} \leq 3 \epsilon$ when $\epsilon<1$.
Finally, we have
\begin{equation}
    \left \| \Zopt {\mathbf{\Delta}^1}^T (0.5 \mathbf{H} + \mathbf{H}_2) \right \|_{F} \leq 0.5 \left \| \mathbf{\Delta}^{1} \right \|_{F} \left \| \mathbf{H} \right \|_{F} + \left \| \mathbf{\Delta}^{1} \right \|_{F} \left \| \mathbf{H}_{2} \right \|_{F} \leq \frac{3 \epsilon^2}{2} + \frac{9 \epsilon^3}{8(1 - \epsilon^2 - 2 \epsilon)^{3/2}}.
    \label{eq:zp4}
\end{equation}

Combining~\eqref{eq:zp1}-\eqref{eq:zp4}, we easily get $\left \| O((\mathbf{\Delta}^1)^2) \right \|_{F} \leq C(\epsilon) \epsilon^2$ with $C(\epsilon)$ as defined in \eqref{convradius}.
Including the $C(\epsilon) \epsilon^2$ term in \eqref{upb1}, the effective convergence rate in Lemma~\ref{lem:w} is $q + C(\epsilon) \epsilon$ with the dominant $q:=   \max_k \|  \D_k \Zopt \tilde{\D}_k\|_2$. 
Thus, $q + C(\epsilon) \epsilon < 1$ suffices for linear convergence or $\epsilon < (1-q)/C(\epsilon)$. Since $C(\epsilon)$ is monotone increasing in $0 \leq  \epsilon < \sqrt{2}-1$ (the upper bound comes from the aforementioned Taylor Series convergence conditions) with $C(\sqrt{2}-1)=\infty$, $\exists \epsilon \in [0, \sqrt{2}-1)$ for which $C(\epsilon) \epsilon = 1-q$. This would be $\epsilon_2$ (largest permissible $\epsilon$) for the operator update step. 
It is easy to see that this $\epsilon_2$ is equivalently obtained by maximizing $f(\epsilon_0)$ in $[0, \sqrt{2}-1)$, where $f(\epsilon_0) = \min \begin{pmatrix}
(1 - q)/C(\epsilon_0),\epsilon_0
\end{pmatrix}$.
Note that we ignored the higher order effects in our Assumptions,
since $C(\epsilon) \epsilon$ is negligible for sufficiently small $\epsilon$, where the effective convergence rate is approximately $q:=   \max_k \|  \D_k \Zopt \tilde{\D}_k\|_2$.

In the case of Remark~\ref{rm1b} for Theorem~\ref{thm:main2}, the form of $\epsilon_1$ remains the same as above.
The proof of Lemma~\ref{lem:z_not_orth} showed that the matrix inverse series converges when the perturbation term satisfies $ \| \mathbf{G}^{-1} \Zopt \boldsymbol{\Delta}^{1^T}\|_2<1$, or $\epsilon < \kappa^{-2}(\Zopt)$ suffices. Similarly, the bounds for the other series terms also depend on $\kappa(\Zopt)$. Clearly, as $\kappa(\Zopt) \to 1$, we approach Assumption $(A_4)$ for which $\epsilon_2$ takes the same form as in Proposition~\ref{convradius}.
The limit for $\epsilon_2$ in Remark~\ref{rm1b} holds for the distributions in Proposition~\ref{conj} because $\max_{1\leq k \leq n} \| \D_k \Zopt \tilde{\D}_k \|_2 \to \sqrt{(s-1)/(n-1)}$ (see \eqref{glim}) almost surely as $N \to \infty$.

\section{Distributions in Section~\ref{convradiussection}} \label{app2}
Various distributions of $\Zopt$ lead to interesting behavior for $\epsilon_1 = 0.5 \min_{1\leq j \leq N} \beta \left( \col{\Z}^*/\|\col{\Z}^*\|_{2} \right)$.
Here, we discuss 
example distributions
and the corresponding behavior of $\epsilon_1$, which we show to be lower bounded by $O(1/\sqrt{s})$.
The distributions below satisfy the conditions in Proposition~\ref{conj} (i.e., the column supports of $\Zopt$ of cardinality $s$ are drawn independently and uniformly at random, and the non-zero entries are i.i.d. with mean zero and variance $n/sN$) to ensure good convergence rate properties.
\begin{enumerate}
    \item  The non-zeros are random signs scaled by $\sqrt{n/sN}$, and ``+" and ``-" are equally probable $\Rightarrow \epsilon_1 = 0.5/\sqrt{s}$.
    \item Non-zeros are uniformly distributed in $[-b,-c]\cup [c,b]$ with 
    $b>c > 0$.
    When $c=\sqrt{3n/7sN}$ and $b=\sqrt{12n/7sN}$, then  $\epsilon_1 \geq  0.25/\sqrt{s}$.
    \item Non-zeros are drawn from the density $p(z)= k e^{-a|z|}$ when $c \leq |z|\leq b$ and $p(z)=0$ otherwise, with $k=0.5 a e^{ac}/(1-e^{-a(b-c)})$ and $b > c > 0$ and $a>0$. For a given $K>1$, $\beta = \log K/(K-1)$, $a=\sqrt{(2-K\beta^2)sN/n}$, $c=\beta/a$, and $b=Kc$ $\Rightarrow$ $\epsilon_1 \geq 0.5/K\sqrt{s}$.
\end{enumerate}
The nonzeros of $\Zopt$ above are assumed to be upper bounded (in practice, the bound is determined by the peak physical intensity in the signals considered) and lower bounded (determined by numerical precision).

We briefly show the $\epsilon_1$ bounds for the examples above. When the non-zeros of $\Zopt$ are random signs scaled by $\sqrt{n/sN}$, it is obvious that $\epsilon_1 = 0.5/\sqrt{s}$.

When the non-zeros are uniformly distributed with $p(z)=0.5/(b-c)$ for $ z \in [-b,-c]\cup [c,b]$ with $b>c > 0$,
then clearly $E[z] = 0$.
The variance of the distribution is $(b^2 + c^2 + bc)/3$. Setting this to the required value of $n/sN$ yields $c^2 +bc$ $ + b^2 - (3n/sN) = 0$. Solving the quadratic equation for $c$ yields a root $0.5 \begin{pmatrix}
-b + \sqrt{(12n/sN)-3b^2}
\end{pmatrix}$, which is nonnegative when $(12n/sN)-3b^2 > b^2$ (i.e., $b < \sqrt{3n/sN}$).
Moreover, $c<b$ implies $b > \sqrt{n/sN}$. 
Then the distributions with $b \in \begin{pmatrix}
\sqrt{n/sN},\sqrt{3n/sN}
\end{pmatrix}$ and $c = 0.5 \begin{pmatrix}
-b + \sqrt{(12n/sN)-3b^2}
\end{pmatrix}$ readily satisfy
$\beta (\col{\Z}^*) \geq c$ and $\|\col{\Z}^*\|_{2} \leq b\sqrt{s}$.
Thus, clearly $\epsilon_1 \geq 0.5c/b\sqrt{s}$. For the special case $b=\sqrt{12n/7sN}$ and $c=\sqrt{3n/7sN}$$\Rightarrow \epsilon_1 \geq  0.25/\sqrt{s}$.

When the nonzeros are drawn from $p(z)= k e^{-a|z|}$ for $c \leq |z|\leq b$ and $p(z)=0$ otherwise, with $k=0.5 a e^{ac}/(1-e^{-a(b-c)})$ 
and $b > c > 0$
and $a>0$, clearly $E[z] = 0$ and the variance is $2a^{-2} +$ $ 2a^{-1}\begin{bmatrix}
(ce^{-ac}-be^{-ab})/(e^{-ac}-e^{-ab})
\end{bmatrix} +$ $ \begin{bmatrix}
(c^{2}e^{-ac}-b^{2}e^{-ab})/(e^{-ac}-e^{-ab})
\end{bmatrix}$. Setting the variance to $n/sN$ yields a nonlinear equation in $a$, $b$, and $c$, with many solutions.
To extract one set of solutions, we set $ce^{-ac} = be^{-ab}$ and $b=Kc$ for some $K>1$, which implies $ac = \log K/(K-1) \triangleq \beta$. Substituting these in the variance equation simplifies it to $n/sN = 2a^{-2} - bc = 2a^{-2} - K \beta^{2} a^{-2}$.
Thus, $a=\sqrt{(2-K\beta^2)sN/n}$ with $c=\beta/a$ and $b=Kc$ in this case.
We then easily get $\epsilon_1 \geq 0.5c/b\sqrt{s} = 0.5/K\sqrt{s}$.

\section{Proof of Lemma~\ref{sparsecodegeneral}} \label{app3}

Each column $\col{\Z}^1$ ($1\leq j \leq N$) of the sparse coefficients matrix $\Z^1$ satisfies
\begin{align}
\col{\Z}^1 \stackrel{}{=} H_s( & \bW^{0}  \col{\bP}) \stackrel{}{=} H_s(\Wopt \col{\bP} + \E^{0} \col{\bP}) \stackrel{(A_1)}{=} H_s(\col{\Z}^* + \E^{0} \col{\bP}) \stackrel{}{=} \mathbf{\Gamma}_j^1 \odot \mathbf{\Gamma}_j^* \col{\Z}^* + \mathbf{\Gamma}_j^1 \E^{0} \col{\bP}, \label{eq:zh1}
\end{align}
where $\E^0 =\bW^0 - \Wopt$ and $\mathbf{\Gamma}_j^1$ is a diagonal matrix with a one at the $(i,i)$th entry when $i \in S(\col{\Z}^1)$ and zero otherwise. Matrix $\mathbf{\Gamma}_j^*$ is similarly defined with respect to $S(\col{\Z}^*)$, and ``$\odot$" denotes element-wise multiplication.

It follows that $\col{\Z}^1 - \col{\Z}^*$ $= \begin{bmatrix}
\begin{pmatrix}
\mathbf{\Gamma}_j^1 \odot \mathbf{\Gamma}_j^* 
\end{pmatrix} - \mathbf{\Gamma}_j^* 
\end{bmatrix} \col{\Z}^* + \mathbf{\Gamma}_j^1 \E^{0} \col{\bP}$, where the two summands have disjoint supports because $\begin{bmatrix}
\begin{pmatrix}
\mathbf{\Gamma}_j^1 \odot \mathbf{\Gamma}_j^* 
\end{pmatrix} - \mathbf{\Gamma}_j^* 
\end{bmatrix}$ is diagonal with zeros and with ``-1" only for the portion of the support of $\col{\Z}^*$ left out in $S(\col{\Z}^1)$. Therefore, we have
\begin{equation}
\begin{Vmatrix}
\col{\Z}^1 - \col{\Z}^*
\end{Vmatrix}_{2}^{2} = \begin{Vmatrix}
\begin{bmatrix}
\begin{pmatrix}
\mathbf{\Gamma}_j^1 \odot \mathbf{\Gamma}_j^* 
\end{pmatrix} - \mathbf{\Gamma}_j^* 
\end{bmatrix} \col{\Z}^*
\end{Vmatrix}_{2}^{2} + \begin{Vmatrix}
\mathbf{\Gamma}_j^1 \E^{0} \col{\bP}
\end{Vmatrix}_{2}^{2}.
\label{eq:zh2}    
\end{equation}
Let $\mathbf{A} \triangleq \E^0 \bP$. To simplify and bound \eqref{eq:zh2}, 
we first consider the case when only one element, say $ k \in S(\col{\Z}^*)$ was left out in $S(\col{\Z}^1)$. Suppose that in its place, we have a new entry $p \in S(\col{\Z}^1)$ with $p \notin S(\col{\Z}^*)$.
Then, we must have 
\begin{equation}
    \begin{vmatrix}
0 + \mathbf{A}_{(p,j)}
\end{vmatrix} \geq \begin{vmatrix}
\Zopt_{(k,j)} + \mathbf{A}_{(k,j)}
\end{vmatrix} \geq \begin{vmatrix}
\Zopt_{(k,j)}
\end{vmatrix} - \begin{vmatrix}
\mathbf{A}_{(k,j)}
\end{vmatrix},
\label{eq:zh3}
\end{equation}
where the first inequality is necessary for the $p$th entry to swap with the $k$th entry in the support and the second inequality is the reverse triangle inequality.
Thus, we have
\begin{equation}
    \begin{vmatrix}
\Zopt_{(k,j)}
\end{vmatrix} \leq \begin{vmatrix}
\mathbf{A}_{(k,j)}
\end{vmatrix} + \begin{vmatrix}
\mathbf{A}_{(p,j)}
\end{vmatrix}.
\label{eq:zh4}
\end{equation}
Note that this holds even if $\mathbf{A}_{(p,j)} = 0$ and $\mathbf{A}_{(k,j)} = - \Zopt_{(k,j)}$, i.e., only the $k$th entry is left out of $S(\col{\Z}^1)$ without a new nonzero ($p$th) entry.
Using these results, \eqref{eq:zh2} can be readily simplified for this case as
\begin{align}
\begin{Vmatrix}
\col{\Z}^1 - \col{\Z}^*
\end{Vmatrix}_{2}^{2} & = \begin{vmatrix}
\Zopt_{(k,j)}
\end{vmatrix}^{2} + \begin{Vmatrix}
\mathbf{\Gamma}_j^1 \mathbf{A}_{(.,j)}
\end{Vmatrix}_{2}^{2} \nonumber \\ & \stackrel{\eqref{eq:zh4}}{\leq} \begin{vmatrix}
\mathbf{A}_{(k,j)}
\end{vmatrix}^2 + \begin{vmatrix}
\mathbf{A}_{(p,j)}
\end{vmatrix}^2 + 2 \begin{vmatrix}
\mathbf{A}_{(k,j)}
\end{vmatrix} \begin{vmatrix}
\mathbf{A}_{(p,j)}
\end{vmatrix} + \begin{Vmatrix}
 \mathbf{A}_{(.,j)}
\end{Vmatrix}_{2}^{2} - \begin{Vmatrix}
\left ( \Id - \mathbf{\Gamma}_j^1 \right ) \mathbf{A}_{(.,j)}
\end{Vmatrix}_{2}^{2} \nonumber \\
& \leq \begin{vmatrix}
\mathbf{A}_{(k,j)}
\end{vmatrix}^2 +  \begin{vmatrix}
\mathbf{A}_{(p,j)}
\end{vmatrix}^2 + 2 \max\begin{pmatrix}
\begin{vmatrix}
\mathbf{A}_{(k,j)}
\end{vmatrix}^2,\begin{vmatrix}
\mathbf{A}_{(p,j)}
\end{vmatrix}^2
\end{pmatrix}
+ \begin{Vmatrix}
 \mathbf{A}_{(.,j)}
\end{Vmatrix}_{2}^{2} - \begin{Vmatrix}
\left ( \Id - \mathbf{\Gamma}_j^1 \right ) \mathbf{A}_{(.,j)}
\end{Vmatrix}_{2}^{2}   \nonumber \\
& =   \begin{vmatrix}
\mathbf{A}_{(p,j)}
\end{vmatrix}^2 + 2 \max\begin{pmatrix}
\begin{vmatrix}
\mathbf{A}_{(k,j)}
\end{vmatrix}^2,\begin{vmatrix}
\mathbf{A}_{(p,j)}
\end{vmatrix}^2
\end{pmatrix}
+ \begin{Vmatrix}
 \mathbf{A}_{(.,j)}
\end{Vmatrix}_{2}^{2} - \begin{Vmatrix}
\left ( \Id - \tilde{\mathbf{\Gamma}}_j^1 \right ) \mathbf{A}_{(.,j)}
\end{Vmatrix}_{2}^{2}. 
\label{eq:zh5}
\end{align}
The last equality above follows because $\left ( \Id - \mathbf{\Gamma}_j^1 \right ) \mathbf{A}_{(.,j)}$ includes $\mathbf{A}_{(k,j)}$ as a nonzero entry, and $\tilde{\mathbf{\Gamma}}_j^1$ is the same as $\mathbf{\Gamma}_j^1$ except that its $(k,k)$th entry is also $1$.

In \eqref{eq:zh5}, $\begin{Vmatrix}
 \mathbf{A}_{(.,j)}
\end{Vmatrix}_{2}^{2} - \begin{Vmatrix}
\left ( \Id - \tilde{\mathbf{\Gamma}}_j^1 \right ) \mathbf{A}_{(.,j)}
\end{Vmatrix}_{2}^{2} 
= \begin{Vmatrix}
\tilde{\mathbf{\Gamma}}_j^1 \mathbf{A}_{(.,j)}
\end{Vmatrix}_{2}^{2}$ $\leq \begin{Vmatrix}
\mathbf{A}_{(.,j)}
\end{Vmatrix}_{2}^{2}$. The first two summands in \eqref{eq:zh5} are bounded by $\begin{Vmatrix}
\mathbf{A}_{(.,j)}
\end{Vmatrix}_{2}^{2}$ and $2 \begin{Vmatrix}
\mathbf{A}_{(.,j)}
\end{Vmatrix}_{2}^{2}$, respectively. Thus, when one element of the true support is misestimated in each column of $\Z^1$, we have
\begin{equation}
\| \Z^1 - \Zopt \|_F^2 = \sum_{j=1}^{N} \begin{Vmatrix}
\col{\Z}^1 - \col{\Z}^*
\end{Vmatrix}_{2}^{2} \leq 4 \sum_{j=1}^N \begin{Vmatrix}
\mathbf{A}_{(.,j)}
\end{Vmatrix}_{2}^{2} = 4 \begin{Vmatrix}
\E^0 \bP
\end{Vmatrix}_{F}^{2} \leq 4 \begin{Vmatrix}
\E^0 
\end{Vmatrix}_{F}^{2} \begin{Vmatrix}
\bP
\end{Vmatrix}_{2}^{2} \stackrel{(A_1)}{=} 4 \begin{Vmatrix}
\E^0 
\end{Vmatrix}_{F}^{2} \leq 4 \epsilon^2.
\label{eq:zh6}    
\end{equation}
This proves \eqref{errorsparse} for the case when (at most) one entry of the support of each $\col{\Z}^*$ is wrongly estimated (left out) in $\col{\Z}^1$.
In the general case, when multiple elements of the support of $\col{\Z}^*$ may be left out in $S(\col{\Z}^1)$, each such element can be paired with a corresponding ``new" element in  $S(\col{\Z}^1)$, and
\eqref{eq:zh4} holds for each such pair.\footnote{The elements left out of the support of $\col{\Z}^*$ can be paired with ``new" elements in $S(\col{\Z}^1)$ one by one, i.e., no overlaps between the pairs. If multiple new elements satisfy \eqref{eq:zh4}, the pairing picks the one with the smallest magnitude.}
The proof in this general case is similar to the aforementioned case, except that there would be summations over the left out or new indices in various equations.
For example, the first summand $\begin{vmatrix}
\mathbf{A}_{(p,j)}
\end{vmatrix}^2$ in \eqref{eq:zh5} would include a summation over all ``new" indices $p$ in $S(\col{\Z}^1)$. However, this summation is still bounded by $\begin{Vmatrix}
\mathbf{A}_{(.,j)}
\end{Vmatrix}_{2}^{2}$. 
Similarly, the second summand in \eqref{eq:zh5} would be summed over the number of (disjoint) pairs, which is again bounded by $2 \begin{Vmatrix}
\mathbf{A}_{(.,j)}
\end{Vmatrix}_{2}^{2}$. Thus, $\left \| \col{\Z}^1 - \col{\Z}^* \right \|_{2}^{2} \leq 4 \begin{Vmatrix}
\mathbf{A}_{(.,j)}
\end{Vmatrix}_{2}^{2} $ holds generally (including when the true support is correctly estimated in $\col{\Z}^1$).
Therefore, a bound as in \eqref{eq:zh6} holds in the general case.
\end{appendices}

\bibliographystyle{imaiai}
\bibliography{dbib}

\end{document}